\documentclass[letterpaper,10pt]{article}
\usepackage[margin=1in]{geometry}

\usepackage{tikz}
\usepackage{amsmath}

\usepackage{graphicx,epsfig,graphics,epsf}
\usepackage{pgfplots}
\pgfplotsset{compat=1.5.1}
\usepackage{tikz}
\usepackage{float,amsmath,amsfonts,amssymb,amsthm}
\usepackage{marginnote}
\usepackage{tkz-graph}
\usepackage{color}
\usepackage{algpseudocode}
\usepackage{algorithmicx}
\usepackage{algorithm}
\newcommand{\norm}[1]{\left\lVert#1\right\rVert}
\usepackage{bbm}
\usepackage{bm}
\usepackage[inline]{enumitem}

\usepackage{afterpage}

\usetikzlibrary{shapes,decorations,arrows,calc,arrows.meta,fit,positioning,backgrounds,automata}
\usepackage{forest}

\usepackage{wrapfig}
\usepackage[normalem]{ulem}

\usepackage{tcolorbox}
\usepackage{mathtools}
\usepackage{wasysym}
\usepackage{xfrac}
\tcbuselibrary{listings,breakable}

\usepackage{tabularx}
\usepackage{array}
\usepackage{amsthm}
\DeclarePairedDelimiter{\ceil}{\lceil}{\rceil}
\DeclarePairedDelimiter{\abs}{\lvert}{\rvert}

\DeclarePairedDelimiterX{\p}[1]{(}{)}{#1}
\DeclarePairedDelimiterX{\br}[1]{[}{]}{#1}

\newcommand{\s}[1]{\left\{ #1\right\}}
\usepackage{hyperref, bookmark}

\usepackage{hyperref}

\usepackage{pifont}

\usepackage{thmtools}
\usepackage{thm-restate}
\usepackage{cleveref}

\usepackage{subfigure}
\usepackage{fancyhdr}

\newcommand{\E}[2][]{\ifthenelse{\equal{#1}{}}{\displaystyle \mathop{\mathbb{E}}\left[#2 \right]}{\underset{#1}{\displaystyle \mathop{\mathbb{E}}}\left[#2 \right]}}
\newcommand{\reg}[2][]{\ifthenelse{\equal{#1}{}}{\mathbb{E}\left[#2 \right]}{R_{#1}\left(#2 \right)}}

\newcommand{\empiricalRisk}{L_{\trainingData{}}}
\newcommand{\trueRisk}{L_{\sampleDist}}

\renewcommand{\Pr}[2][]{\ifthenelse{\equal{#1}{}}{\mathbb{P}\left[#2 \right]}{\underset{#1}{\mathbb{P}}\left[#2 \right]}}

\newcommand{\opt}{*}

\newcommand{\subsetB}{B}
\newcommand{\Ndim}{Ndim}

\renewcommand{\vec}{\bm}

\newcommand{\mean}[1]{\bm{\mu}_{\mathsf{#1}}}

\newcommand{\adversarialRiskFunc}{\rho}
\newcommand{\naturalRiskFunc}{\nu}

\newcommand{\natarajanDimension}{\mathrm{NDim}}

\newcommand{\standardDeviation}{\sigma}

\newcommand{\sgn}{\mathrm{sgn}}

\newcommand{\given}{\:\middle\vert\:}
\newcommand{\transpose}{\mathsf{T}}

\newcommand{\sample}{\vec{x}}
\newcommand{\sampleIntro}{\sample}
\newcommand{\perturbationIntro}{\perturbation}

\newcommand{\trainingData}[1]{S_{#1}}

\newcommand{\unlabeledTrainingData}[1]{Q_{#1}}

\newcommand{\scalarParameter}{\sigma}

\newcommand{\scalarParameterVar}{\vartheta_{\unlabeledTrainingData{},\noiseDist}}

\newcommand{\marker}{*}
\newcommand{\markerAlt}{\dagger}

\newcommand{\numSamples}{N}

\newcommand{\queryDist}{\mathcal{P}}
\newcommand{\sampleDist}{\mathcal{D}}

\newcommand{\error}{\eta}

\newcommand{\confidence}{\delta}

\newcommand{\realNumbers}{\mathbb{R}}
\newcommand{\naturalNumbers}{\mathbb{N}}

\newcommand{\dimension}{d}

\newcommand{\labelElement}[1]{y_{#1}}

\newcommand{\loss}{\ell}

\newcommand{\sampleComplexity}[1]{m_{#1}}

\newcommand{\inputSpace}{\mathcal{X}}

\newcommand{\outputSpace}{\mathcal{Y}}
\newcommand{\outputSpaceInstance}{\outputSpace}

\newcommand{\hypothesisClass}{\mathcal{H}}
\newcommand{\hypothesisClassRestricted}{\mathcal{H}_{\unlabeledTrainingData{}}}

\newcommand{\hypothesisClassSmooth}{\mathcal{H}_{\unlabeledTrainingData{},\noiseDistParametrized}}

\newcommand{\drawnFrom}{\sim}
\newcommand{\learningAlgorithm}{\mathcal{A}}

\newcommand{\hyp}{\mathrm{h}}
\newcommand{\hypContinuous}{\mathrm{g}}

\newcommand{\placeholder}{z}

\newcommand{\approximation}{\mathsf{app}}
\newcommand{\estimation}{\mathsf{est}}

\newcommand{\placeholderFuncAltAlt}{\nu}

\newcommand{\placeholderFunc}{\psi}

\newcommand{\VCdimension}{d}

\newcommand{\noiseDist}{\mathsf{n}}
\newcommand{\func}{f}

\newcommand{\noiseDistParametrized}{\noiseDist_{\scalarParameter}}
\newcommand{\noiseDistInstance}[1]{\noiseDist_{#1,\scalarParameter}}
\newcommand{\noiseInstance}{\eta}

\newcommand{\constantClassificationSet}{\mathcal{B}_{\perturbationNormBound,p}}

\newcommand{\perturbation}{\Delta}

\newcommand{\smoothingFunction}{\mathcal{S}_{\noiseDistParametrized}}

\newcommand{\predictionRule}{f}

\newcommand{\zeroOne}{0\text{-}1}

\newcommand{\perturbationNormBound}{\epsilon}
\newcommand{\convolutionOperator}{\ast}

\newcommand{\definedAs}{\triangleq}

\newcommand{\zerosDimension}{\vec{0}_{\dimension}}
\newcommand{\inlineMatrix}[1]{\begin{bsmallmatrix} #1 \end{bsmallmatrix}}

\hypersetup{
    colorlinks=true,
    linkcolor=black,
    filecolor=black,
    urlcolor=blue,
    citecolor=black
}

\graphicspath{ {writeupfigures/} }
\definecolor{solarBlue}{RGB}{44,54,67}
\definecolor{solarText}{RGB}{120,200,120}
\definecolor{solarText}{RGB}{255,200,63}
\definecolor{mathColor}{RGB}{39,222,222}

\newtheorem{theorem}{Theorem}

\newtheorem*{pac}{Sample Complexity of PAC learning}

\newtheorem{conjecture}[theorem]{Conjecture}
\newtheorem{corollary}[theorem]{Corollary}

\theoremstyle{plain}

\theoremstyle{definition}
\newtheorem{definition}[theorem]{Definition}

\newtheorem*{ERM}{Empirical Risk Minimization (ERM)}

\newtheorem*{noiseDistribution}{Noise Distribution}

\newtheorem*{shatterVC}{Hypothesis Shattering}

\newtheorem*{biasComplexityTradeoff}{Bias-Complexity Tradeoff}
\newtheorem*{learnability}{Realizablity and Learnability}

\newcommand{\stkout}[1]{\ifmmode\text{\sout{\ensuremath{#1}}}\else\sout{#1}\fi}
\renewcommand{\boldsymbol}{\vec}

\usepackage{graphicx}
\usepackage{xcolor}
\usepackage{etoolbox}
\usepackage{booktabs}
\usepackage{subfigure}

\algnewcommand{\LineComment}[1]{\State \(\triangleright\) #1}

\usepackage[utf8]{inputenc}
\crefname{thmenum}{theorem}{theorems}
\usepackage{authblk}

\begin{document}

\date{}

\title{\Large \bf Analyzing Accuracy Loss in Randomized Smoothing Defenses}

\author{Yue Gao \thanks{Harrison Rosenberg and Yue Gao made equal contributions to this work.}}
\author{Harrison Rosenberg}
\author{Kassem Fawaz}
\author{Somesh Jha}
\author{Justin Hsu}
\affil{\tt{gy@cs.wisc.edu hrosenberg@ece.wisc.edu kfawaz@wisc.edu jha@cs.wisc.edu justhsu@cs.wisc.edu}}

\affil{University of Wisconsin--Madison}

\maketitle

\begin{abstract}
Recent advances in machine learning (ML) algorithms, especially deep neural
networks (DNNs), have demonstrated remarkable success (sometimes exceeding
human-level performance) on several tasks, including face and speech
recognition. However, ML algorithms are vulnerable to \emph{adversarial
attacks}, such test-time, training-time, and backdoor attacks. In test-time
attacks an adversary crafts adversarial examples, which are specially crafted
perturbations imperceptible to humans which, when added to an input example,  force a machine learning model to misclassify the given input example. Adversarial examples are a concern when deploying ML
algorithms in critical contexts, such as information security and autonomous
driving. Researchers have responded with a plethora of defenses. One
promising defense is \emph{randomized smoothing} in which a classifier's prediction is
smoothed by adding random noise to the input example we wish to classify. In this paper, we
theoretically and empirically explore randomized smoothing. We investigate the
effect of randomized smoothing on the feasible hypotheses space, and show that
for some noise levels the set of hypotheses which are feasible shrinks due to
smoothing, giving one reason why the natural accuracy drops after smoothing. To
perform our analysis, we introduce a model for randomized smoothing which
abstracts away specifics, such as the exact distribution of the noise. We
complement our theoretical results with extensive experiments.
\end{abstract}

\section{Introduction}
\label{sec:intro}
%
%
Recent advances in machine learning (ML), such as deep neural networks (DNNs), have paralleled  or exceeded human-level performance over a wide range of computer vision tasks, including image classification and visual recognition~\cite{DBLP:journals/nature/LeCunBH15}. In recent years, researchers have discovered that DNNs are vulnerable to various 
attacks (e.g. test-time, training-time, and backdoor attacks). In test-time attacks, which is the topic of this paper, an adversary crafts \emph{adversarial examples}--- a cleverly crafted human-imperceptible perturbation $\perturbationIntro$ that, when added to an input $\sampleIntro$, can cause the network to output an incorrect prediction for $\sampleIntro+\perturbationIntro$~\cite{DBLP:journals/corr/SzegedyZSBEGF13}. These attacks raise severe safety and security concerns, especially when ML models are deployed in critical systems, such as face-recognition based bio-metrics~\cite{DBLP:conf/bmvc/ParkhiVZ15}, malware detection \cite{DBLP:conf/malware/SaxeB15}, medical imaging~\cite{DBLP:journals/tmi/GreenspanGS16}, and autonomous driving \cite{DBLP:journals/corr/HuvalWTKSPARMCM15}.

%
In response to these attacks, researchers have proposed several defenses. These methods can be broadly divided into \emph{empirical} and \emph{certified} defenses. Adversarial training~\cite{DBLP:conf/iclr/KurakinGB17,DBLP:conf/iclr/MadryMSTV18} is an example of an empirical defense which has survived the recent onslaught of attacks~\cite{DBLP:conf/sp/Carlini017, DBLP:conf/iclr/MadryMSTV18}. Adversarial training borrows techniques from robust optimization to minimize a worst-case loss in a certain
set of perturbations. This defense, however, has the shortcoming of yielding a computationally expensive training procedure. Adversarial training scales poorly as datasets and DNNs become larger. On the other hand, certified defenses~\cite{DBLP:conf/sp/LecuyerAG0J19, DBLP:conf/icml/CohenRK19} guarantee that the classifier's prediction is stable within some set around the input $\sampleIntro$, often an $\ell_p$ norm ball around $\sampleIntro$. 

%
%
Randomized smoothing~\cite{DBLP:conf/sp/LecuyerAG0J19, DBLP:conf/icml/CohenRK19} is one popular approach for certified defense. This approach converts a base classifier into a \emph{smoothed} version through sampling its output from a noisy input distribution. The procedure typically involves sampling from a noise distribution centered about the input example, passing the samples through the classifier, and choosing the label using a majority vote. Unlike many prior certified defenses~\cite{raghunathan2018certified}, randomized smoothing can provide certification in $\ell_p$ norm with lower computational overhead than adversarial training, and thus scaling up to larger and more challenging datasets such as ImageNet~\cite{DBLP:conf/icml/CohenRK19}.

%
%
While showing promise in making models certifiably robust, several aspects of randomized smoothing require further exploration.
First, according to the results from literature, improved certification almost always comes at a cost of reduced accuracy of the smoothed model over natural examples.
Second, the relationship between the noise level in smoothing and the resulting accuracy of the smoothed model in \emph{adversarial} settings is not well-understood.
Finally, the empirical evaluation of randomized smoothing typically requires an unexplained training procedure with noise augmented dataset~\cite{DBLP:conf/icml/CohenRK19,DBLP:conf/sp/LecuyerAG0J19}. It is clear that 
certification is a straightforward outcome of smoothing but not of noise augmentation,
hence it is interesting to study the relationship between noise augmentation and randomized smoothing. 


In this paper, we perform an in-depth exploration of defenses based on randomized smoothing, and its relation to noise augmentation during training. We first provide a generalized formulation for randomized smoothing, including the noise augmented training procedure that precedes randomized smoothing. This formulation generalizes the smoothing approaches of Cohen et al.~\cite{DBLP:conf/icml/CohenRK19} and L{\'{e}}cuyer et al.~\cite{DBLP:conf/sp/LecuyerAG0J19}. Based on that, we analyze the impact of randomized smoothing and its noise augmented training procedure on the set of realizable hypotheses. Our main result identifies a critical noise threshold, beyond which the set of hypotheses realizable on training data after smoothing is a strict subset the set of hypotheses realizable on training data prior to smoothing. 

This result holds a significant implication on the performance of randomized smoothing. A hypothesis class of reduced size implies that it is easier for the training procedure to converge to the best hypothesis in the class (by requiring fewer training samples). 
It is not necessarily the case that all hypotheses in a hypothesis class which are realizable before smoothing are still realizable after smoothing. The implication is clear; the test accuracy over both natural and adversarial examples may decrease. Perhaps, this is an explanation of reduced accuracy of ML models after smoothing.

As such, one can view randomized smoothing as a trade-off between certification and accuracy. A noise distribution with little statistical dispersion yields a weak robustness certificate, but is less likely to have an adverse effect on natural accuracy. On the other hand, a noise distribution with large statistical dispersion improves the certification bounds and decreases estimation error at the risk of reducing test accuracy.

We conduct extensive experiments on noise augmentation and randomized smoothing to support our conclusions. First, the recent adversarial attack on randomized smoothing proposed by Salman et al.~\cite{smoothing-pgd} makes it possible to investigate randomized smoothing in adversarial settings. We empirically observe that noise augmentation alone can already provide adversarial robustness, and we also empirically observe smoothing is not effective without noise augmentation.
Then, we study the relationship between training accuracy, test accuracy, and the statistical dispersion of our noise distribution with a randomly labeled dataset~\cite{random-labeling}. Consistent with our main theoretical result, adding more noise pushes, in an informal sense, training accuracy closer to test accuracy at a considerable cost to both natural and adversarial accuracy.
Our main contributions are summarized below:
\begin{itemize}[leftmargin=*]
	\item {\it Importance of Noise Augmentation.} We draw attention to the noise augmented training procedure, which was earlier only empirically motivated through experiments. We provide a generalized formulation for randomized smoothing.  Our results suggest that the accuracy on a training set for a model trained with noise augmentation, when compared to a model trained without noise augmentation, is more indicative of natural and adversarial test accuracy of the same model after the randomized smoothing is applied.
	
	\item {\it Noise Augmentation Reduces the Realizable Hypothesis Space.} We show that noise augmented training, when compared to training without noise augmentation,  yields a set of a smaller set of hypotheses realizable on training data. This result has two implications on the performance of randomized smoothing:
		\begin{itemize}[leftmargin=*]
			\item First, the resulting smoothed classifier achieves improved robustness while losing its expressiveness. It is well documented in literature that this loss of expressiveness can contribute to a drop in accuracy.
			
			\item Second, randomized smoothing reduces the generalization error on the realizable set of hypotheses. While smoothing might come at a cost of test set accuracy, this accuracy becomes closer to the training set accuracy. 
		\end{itemize}
		
	\item {\it Empirical Support for Theoretical Results.} Our evaluation on several datasets and architectures provide empirical support for our theoretical arguments. More importantly, our evaluations with random labels attribute the above implications to noise augmentation, and indicate that smoothing (without training with noise) only complements the accuracy and provides robustness certification.
\end{itemize}

\section{Background}
\label{sec:background}

In this section, we describe the background material and notation necessary for understanding the concepts presented in this paper, such as randomized smoothing. In the rest of this paper, we denote vectors in lowercase bold script, and functions, sets, operators, and scalars in standard script. For a vector $\vec{\placeholder}$, $\placeholder_i$ denotes its $i^{\mathrm{th}}$ element.  The vector of all zeros in $\realNumbers^{\dimension}$ is denoted by $\zerosDimension$. $\realNumbers_+$ denotes the set of positive real numbers.

\subsection{ML Background and Notation}\label{subsec:MLBackground}


We consider a setting where we have some input set $\inputSpace \subseteq \realNumbers^{\dimension}$ and some label set $\outputSpace$. The training sequence is denoted as $\trainingData{} \subseteq \inputSpace \times \outputSpaceInstance{}$, where $\abs{\trainingData{}}=\numSamples$ and $\numSamples$ is a positive integer greater than $2$.  The sequence $\trainingData{}$ is i.i.d. sampled from an unknown probability distribution $\sampleDist$, where the marginal probability distribution of $\inputSpace{}$ is a continuous, non-degenerate function. We denote by $\unlabeledTrainingData{}$ the unlabeled training data.  That is, $\unlabeledTrainingData{} \definedAs \s{\sample : (\sample,\labelElement{}) \in \trainingData{}}$. As necessary, subscripts will be used when enumerating training data; it may be expedient to write $\trainingData{}\definedAs \s{(\sample_1,\labelElement{1}), \hdots, (\sample_{\numSamples},\labelElement{\numSamples})}$. Let $\hypothesisClass$ be a hypothesis class from which a deterministic prediction rule, $\hyp : \inputSpace \to \outputSpace$ is selected. For example, in a binary classification setting, a hypothesis class $\hypothesisClass$ which appears frequently in literature is the set of non-homogeneous linear classifiers in $\realNumbers^\dimension$, i.e. classifiers of the form $\hyp(\sample) = \sgn\s{\vec{w}^{\transpose}\sample + b}$ where $\vec{w} \in \realNumbers^{\dimension},b\in\realNumbers$.   Requiring $\numSamples$ to be a positive integer implies $\abs{\trainingData{}}$ is finite.  Requiring $\numSamples$ to be a positive integer while enforcing continuity and non-degeneracy of the marginal probability distribution of $\inputSpace$ implies that for any pair of samples $(\sample_i,\sample_j) \in \unlabeledTrainingData{}$ such that $i\neq j$, $\sample_i \neq \sample_j$ almost surely.
Note that when $\abs{\outputSpace{}} \geq 3$, $\abs{\unlabeledTrainingData{}} = \numSamples$ and $\natarajanDimension(\hypothesisClass)=\VCdimension$, $\abs{\hypothesisClass}$ is only known to have an upper bound.  In particular, $\abs{\hypothesisClassRestricted}\leq\numSamples^{\VCdimension}\cdot \abs{\outputSpace{}}^{2\VCdimension}$ \cite{shalev2014understanding}.  Consequently, within this paper, we will assume labels are mutually exclusive, and $\hypothesisClass$ is sufficiently expressive such that when $\numSamples \leq \natarajanDimension(\hypothesisClass)$, we have $\abs{\hypothesisClass}=\abs{\outputSpace{}}^{\numSamples}$. 

\begin{ERM}
We borrow much of the notation of this section from Shalev-Shwartz and Ben-David \cite{shalev2014understanding}. 

The true risk of a hypothesis $\hyp$ is defined to be 
\begin{equation}
    \trueRisk(\hyp) \definedAs \Pr[(\sample,\labelElement{})\drawnFrom \sampleDist]{\hyp(\sample) \neq \labelElement{}}
\end{equation}
The goal of learning is to find a hypothesis $\hyp$ for which $\trueRisk$ is minimized.  Unfortunately, the learner does not know $\sampleDist$, but the learner does have access to the training data $\trainingData{}$, an i.i.d. sample from $\sampleDist{}$.  The empirical risk of a hypothesis $\hyp$ is defined to be 
\begin{equation}\label{eq:empiricalRisk}
    \empiricalRisk(\hyp) \definedAs \frac{1}{\numSamples}\sum_{(\sample,\labelElement{}) \in \trainingData{}}\loss_{\zeroOne}\p*{\hyp(\sample), \labelElement{}}
\end{equation}

where $\loss_{\zeroOne}$ is the $0$-$1$ loss function:
\begin{equation}
    \loss_{\zeroOne}(\placeholder,\labelElement{}) = \begin{cases} 0 & \placeholder=\labelElement{} \\ 1 & \placeholder\neq\labelElement{} \end{cases}
\end{equation}

A learning algorithm $\learningAlgorithm_{\hypothesisClass,\loss_{\zeroOne}}$ selects a hypothesis from hypothesis class $\hypothesisClass$ which minimizes empirical risk $\empiricalRisk$.  As the $0$-$1$ loss function is non-convex, our experiments -- and several applications of machine learning -- minimize a convex surrogate loss which upper-bounds the $0$-$1$ loss. We denote by $\learningAlgorithm_{\hypothesisClass,\loss}$ a learner which yields a hypothesis $\hyp \in \hypothesisClass$ that minimizes some convex surrogate $\loss$ function for $0$-$1$ loss. Common convex surrogates for the $0$-$1$ include the cross entropy loss and the hinge loss. More precisely, the learning algorithm $\learningAlgorithm_{\hypothesisClass,\loss}$ can be formalized as:
\begin{equation}
    \learningAlgorithm_{\hypothesisClass,\loss}(\trainingData{}) \definedAs \arg\min_{\hyp \in \hypothesisClass}\s{\sum_{(\sample,\labelElement{}) \in \trainingData{}} \loss(\hyp(\sample),\labelElement{})}.
\end{equation}
 This procedure is known as Empirical Risk Minimization (ERM).
\end{ERM}

\begin{learnability}
We base our definitions of realizability, Probably Approximately Correct (PAC) learnability and Agnostic PAC learnability from Shalev-Shwartz and Ben-David \cite{shalev2014understanding}.

The realizability assumption states that there exists $\hyp^{\opt} \in \hypothesisClass$ such that $\trueRisk(\hyp^{\opt})=0$.

A hypothesis class $\hypothesisClass$ is PAC learnable if there exists a function $\sampleComplexity{\hypothesisClass} : (0,1)^2 \to \naturalNumbers$ and a learning algorithm with the following property: For every $(\error,\confidence) \in (0,1)$, and for every distribution $\sampleDist$ over $\inputSpace{} \times \outputSpace{}$, if the realizable assumption holds with respect to $\hypothesisClass,\sampleDist$, then when running the algorithm on $\numSamples \geq \sampleComplexity{\hypothesisClass}(\error,\confidence)$ i.i.d. examples generated by $\sampleDist$, the algorithm returns $\hyp$ with probability of at least $1-\confidence$, over the choice of examples,
\begin{equation}
\trueRisk(\hyp) \leq \error
\end{equation}

The function $\sampleComplexity{\hypothesisClass}(\error,\confidence)$ is also known as sample complexity.


Throughout this paper, we assume the realizability assumption holds for $\hypothesisClass$.  
\end{learnability}

\begin{biasComplexityTradeoff}
We base our definition and analysis of the bias complexity tradeoff from Shalev-Shwartz and Ben-David \cite{shalev2014understanding}.

The error of an ERM algorithm over hypothesis class $\hypothesisClass$, is often decomposed into two parts, approximation error $\error_{\approximation}$ and estimation error, $\error_{\estimation}$.  Let $\hyp = \learningAlgorithm_{\hypothesisClass,\loss_{\zeroOne}}(\trainingData{})$.  Then 
\begin{align}
    \trueRisk(\hyp) = \error_{\approximation} + \error_{\estimation} \text{ where }&\: \error_{\approximation}= \min_{\hyp' \in \hypothesisClass}\trueRisk(\hyp')\\ &\:\error_{\estimation}=\trueRisk(\hyp)-\error_{\approximation}\label{eq:estimation_error}
\end{align}

The approximation error captures inductive bias.  This term measures how much risk is contributed by selection of our hypothesis class.

The estimation error captures how well the learning algorithm $\learningAlgorithm_{\hypothesisClass,\loss_{\zeroOne}}$ estimates the predictor $\hyp^{\opt} \in \hypothesisClass$ which minimizes true risk $\trueRisk$.  The quality of this estimate depends on the size of the training set $\trainingData{}$, and on the size of the hypothesis class $\hypothesisClass$.

For fixed distribution $\sampleDist$ and hypothesis $\hyp$, this decomposition suggests a tradeoff between estimation error and approximation error.  The term \emph{overfitting} is often used to describe situations in which  approximation error is significantly less than estimation error.  Conversely, \emph{underfitting} is often used to describe situations in which estimation error is significantly less than approximation error.
\end{biasComplexityTradeoff}

\begin{shatterVC}
It will be useful later to reason about the size of the \textit{feasible} $\hypothesisClass$. Informally, a feasible hypothesis is one that satisfies the label assignment for a data sequence; not every hypothesis is a feasible given a data sequence. Formally, we define the feasible hypotheses as a restriction of $\hypothesisClass$ to $\unlabeledTrainingData{}$ (the unlabeled training data), denoted as $\hypothesisClassRestricted$, where  $\hypothesisClassRestricted$ is the set of functions from $\unlabeledTrainingData{}$ to $\outputSpace{}$ which can be derived from $\hypothesisClass$.  That is,
\begin{equation}
  \hypothesisClassRestricted \definedAs \s{ \hyp|_{\unlabeledTrainingData{}} : \hyp \in \hypothesisClass }
\end{equation}
where $\hyp|_{\unlabeledTrainingData{}} : \unlabeledTrainingData{} \to
\outputSpace{}$ is the restriction of $\hyp$ to domain
$\unlabeledTrainingData{}$;

To assist in counting the size of $\hypothesisClassRestricted$, we introduce the concept of multiclass hypothesis shattering from Shalev-Shwartz and Ben-David \cite{shalev2014understanding}. A hypothesis class $\hypothesisClass$ shatters a finite set $\unlabeledTrainingData{}\subset \inputSpace{}$ if there exist two functions $\predictionRule_0,\predictionRule_1 : \trainingData{} \to \outputSpace{}$ such that 
\begin{itemize}
    \item For every $\sample \in \unlabeledTrainingData{}$, $\predictionRule_0(\sample), \predictionRule_1(\sample)$.
    \item For every $\subsetB \subset \unlabeledTrainingData{}$, there exists a function $\hyp \in \hypothesisClass{}$ such that%
    \begin{equation}
        \forall \sample \in \subsetB, \hyp(\sample)=\predictionRule_0(\sample) \text{ and } \forall \sample \in \unlabeledTrainingData{} \setminus \subsetB, \hyp(\sample) = \predictionRule_1(\sample)
    \end{equation}
\end{itemize}

The Natarajan dimension of $\hypothesisClass$, denoted by $\Ndim(\hypothesisClass)$, is the maximal size of a shattered set $\unlabeledTrainingData{} \subset \inputSpace{}$.

Readers may be aware of shattering in the context of Vapnik–Chervonenkis (VC) dimension \cite{vapnik1971uniform}:

 The VC-Dimension of a hypothesis class $\hypothesisClass$, denoted $\VCdimension(\hypothesisClass)$, is the maximal size of a set $\unlabeledTrainingData{} \in \inputSpace{}$ that can be shattered by $\hypothesisClass$.  For example, the set of non-homogenous linear classifiers in $\realNumbers^\dimension$ has VC-dimension $\dimension+1$.
 
 VC-dimension only applies to binary classification.  Natarajan dimension is a generalization of VC-dimension from binary classification to multiclass classification.  Indeed, when $\abs{\outputSpace{}}=2$ and classes are mutually exclusive, Natarajan dimension is VC-dimension exactly.
 
 If $\hypothesisClass$ can shatter sets of arbitrarily large size, we say that $\hypothesisClass$ has infinite Natarajan dimension.  In the binary classification setting, Sine waves in $\realNumbers$, i.e. functions of the form $\sgn\s{\sin(w\cdot x)}$ where $w\in\realNumbers$, is known to have infinite VC-dimension, hence the hypothesis class has infinite Natarajan dimension.


%
%

%



\end{shatterVC}





\subsection{Adversarial ML Background}\label{subsec:AdvMLBackground}

Consistent with the randomized smoothing literature~\cite{DBLP:conf/icml/CohenRK19,DBLP:conf/sp/LecuyerAG0J19}, we address an $\ell_p$-norm bounded adversary which aims to change the classification output. The adversary searches for a perturbation, $\perturbation \in \constantClassificationSet(\sample)$, such that: $\hyp(\sample + \perturbation) \neq \hyp(\sample)$. The set $\constantClassificationSet(\sample)$ is often of the form $\constantClassificationSet(\sample) \definedAs \s{\vec{\placeholder} : \norm{\vec{\placeholder}-\sample}_p \leq \epsilon}$. That is to say $\constantClassificationSet(\sample)$ is some  $\ell_p$-norm bounded set of adversarial perturbations.

Adversarial loss refers to the setting in which risk is measured with respect to inputs perturbed by some adversary.  Given a loss function $\loss$ and $\constantClassificationSet(\sample)$, the adversarial loss for a hypothesis $\hyp$ is denoted by $\adversarialRiskFunc_{\loss}(\hyp)$:
\begin{equation}
    \adversarialRiskFunc_{\loss}(\hyp) \definedAs \E[(\sample,\labelElement{})\drawnFrom \sampleDist]{\max_{\perturbation \in \constantClassificationSet(\sample)} \loss(\hyp(\sample + \perturbation),\labelElement{})}
\end{equation}

Contrast $\adversarialRiskFunc_{\loss}(\hyp)$ with the natural loss $\naturalRiskFunc_{\loss}(\hyp)$:
\begin{equation}
    \naturalRiskFunc_\loss(\hyp) \definedAs \E[(\sample,\labelElement{})\drawnFrom \sampleDist]{ \loss(\hyp(\sample),\labelElement{})}
\end{equation}

A classifier is said to be $\constantClassificationSet(\sample)$-robust if for any input $\sample$, one can obtain a guarantee that the classifier's prediction is constant within $\constantClassificationSet(\sample)$.  That is, given a set $\constantClassificationSet(\sample)$, the classifier $\hyp$ is robust at $\sample$ if, for all $\vec{\placeholder} \in \constantClassificationSet(\sample)$, $\hyp(\sample) = \hyp(\vec{\placeholder})$.  








\section{Roadmap and Key Results}
Before going into the details of our argument, we will provide a roadmap of our paper and a brief overview of our main theorems and key results. In this paper, we present two novel results: 

\begin{itemize}
    \item There exist distributions for which a learned classifier may have $0\%$ error prior to smoothing, but after smoothing, the same classification task is no longer learnable with $0\%$ error.
    
    \item  Natural accuracy can suffer when randomized smoothing is applied to a classifier.
\end{itemize}


To establish the first result, in~\cref{sec:smoothing_def}, we define a novel generalized formulation of randomized smoothing. Utilizing that general formulation, we show in~\cref{thm:impossible} that randomized smoothing may, depending on the underlying sample distribution $\sampleDist{}$, impact the learnability on a training set.  Predictions rendered by randomized smoothing, even predictions made on examples in the training set, depend directly on the choice of noise distribution $\noiseDistParametrized{}$ and the geometry of underlying sample distribution $\sampleDist{}$, but $\sampleDist{}$ is unknown.

Given randomized smoothing has such a dependence on $\sampleDist{}$, we instead analyze noise augmentation.  Noise augmentation is a technique which serves as a proxy for randomized smoothing. Predictions rendered by a classifier trained with noise augmentation instead have a direct dependence on training data $\trainingData{}$ and choice of noise distribution $\noiseDistParametrized{}$.  Noise augmentation only depends on sample distribution $\sampleDist{}$ insofar as $\trainingData{}$ is sampled from $\sampleDist{}$.  Though not mentioned explicitly by name, the noise augmented training procedure frequently appears in randomized smoothing literature ~\cite{DBLP:conf/sp/LecuyerAG0J19, DBLP:conf/icml/CohenRK19}.  This may be because of the reasons mentioned above.


Before arriving at the aforementioned~\cref{thm:impossible}, we further discuss the relationship between randomized smoothing and noise augmentation in~\cref{subsec:smoothingAndNoiseAugmentation}. A small illustrative example of noise augmentation is provided in~\cref{subsec:motivatingExample}.   

The majority of our theory is developed in~\cref{sec:labelClustering} and its implications are discussed in~\cref{sec:implications}. We invoke the Natarajan dimension, a multi-class generalization of Vapnik-Chervonenkis Dimension (VC-Dimension), when discussing the size of our hypothesis classes.  In~\cref{thm:threshold}, we show that beyond a certain noise threshold $\scalarParameterVar{}$, the classification rendered by a model trained with noise augmentation may not provide $0\%$ error on the training set.  We interpret this to mean {\it hypotheses realizable prior to randomized smoothing are not realizable after randomized smoothing}.

After stating our theoretical results, we conduct an extensive experimental evaluation of our claims from~\cref{sec:labelClustering,sec:implications} in~\cref{sec:eval}.  Further details regarding our experiments, including our choice of setup and the design can be found in the experimental section.



\section{Revisiting Randomized Smoothing}
\label{sec:smoothing_def}

Making a classifier more robust to adversarial perturbation has been the subject of recent research. A popular approach to providing computationally inexpensive robustness guarantees is through \emph{randomized smoothing}~\cite{DBLP:conf/icml/CohenRK19,DBLP:conf/sp/LecuyerAG0J19}. These approaches smooth a classifier by sampling its output from a noisy input distribution. The procedure typically involves sampling from a noise distribution centered about the input example, passing an ensemble of such noisy samples to the classifier, then returning the label which the classifier deems to be most probable. 

The randomized smoothing operation yields classifiers which are  $\constantClassificationSet(\sample)$-robust for all $\sample \in \unlabeledTrainingData{}$. In what follows, we formalize a generalized procedure for randomized smoothing. We utilize this formulation to study the impact of randomized smoothing on the feasible set of realizable hypotheses.

\begin{noiseDistribution}
Critical to the randomized smoothing procedure is the noise distribution, $\noiseDistInstance{\mean{}}$. Cohen et al.~\cite{DBLP:conf/icml/CohenRK19} require the distribution to follow an isotropic Gaussian distribution over $\realNumbers^\dimension$. L{\'{e}}cuyer et al.~\cite{DBLP:conf/sp/LecuyerAG0J19} require a joint distribution of $\dimension$ independent samples from a single-variable Laplace Distribution. In this paper, we provide a generalized notion of randomized smoothing which subsumes the noise distributions of L{\'{e}}cuyer et al.~\cite{DBLP:conf/sp/LecuyerAG0J19} and Cohen et al.~\cite{DBLP:conf/icml/CohenRK19}. Before stating the noise distribution, we provide the definition of quasiconcavity:

\begin{definition}
A function $\placeholderFuncAltAlt : \inputSpace \to \realNumbers$ defined a convex subset $U$ of $\realNumbers^\dimension$ is said to be quasiconvex if for all $\sample,\sample' \in U$, and all $\lambda \in [0,1]$,
\begin{equation}
    \placeholderFuncAltAlt(\lambda\sample + (1-\lambda)\sample') \leq \max\s{\placeholderFuncAltAlt(\sample),\placeholderFuncAltAlt(\sample')}
\end{equation}
\end{definition}

In particular, we define a generalized noise distribution as follows. Let $\vec{\placeholder} \in \realNumbers^\dimension$;  $\noiseDistInstance{\mean{}}(\vec{\placeholder})$ is a $\dimension$-dimensional probability density function with the following properties:  
\begin{enumerate}[label=(\roman*)]
    \item Let $\vec{\placeholder} \drawnFrom \noiseDistInstance{\mean{}}$, then $\mean{} = \E{\vec{\placeholder}}$
    \item $\mean{} = \arg\max_{\vec{\placeholder}}\noiseDistInstance{\mean{}}(\vec{\placeholder})$. That is to say, $\mean{}$ is the statistical mode of $\noiseDistInstance{\mean{}}(\vec{\placeholder})$.
    \item $\noiseDistInstance{\mean{}}(\vec{\placeholder}) = \prod_{i=1}^\dimension \noiseDistInstance{{\mu_{i}}}(\placeholder_{i})$ where $\noiseDistInstance{\mu_i}(\placeholder_i) = e^{-\placeholderFunc\p*{\frac{\placeholder_i}{\scalarParameter}}}$ is a symmetric, quasi-concave, non-degenerate function of $\placeholder_i$ from $\realNumbers$ to $(-\infty,+\infty]$ and $\scalarParameter$ is a positive real number.  In other words, $\noiseDistInstance{\mean{}}(\vec{\placeholder})$ is separable and symmetric.
    \item The noise distribution $\noiseDistInstance{\mean{}}$ has measure one.  That is, $1=\int_{\vec{\placeholder} \in \inputSpace{}}\noiseDistInstance{\mean{}}(\vec{\placeholder})\mathrm{d}\vec{\placeholder}$
\end{enumerate}

In \Cref{tab:noiseDistributionExamples}, we show how several common 1-D noise distributions translate into the general noise distribution $\noiseDistInstance{\mu}$.  When translating distributions in more than $1$ dimension, recall that we assumed $\noiseDistInstance{\mean{}}$ to be separable.


\begin{table}[bt]
\caption{Lookup Table for Normal and Laplace Distributions}
\label{tab:noiseDistributionExamples}
\centering
\resizebox{0.6\linewidth}{!}{\begin{tabular}{ccccc}
\toprule
Distribution & \multicolumn{1}{c}{$\func_{X}(x)$} & $\placeholderFunc(\placeholder)$ & $\scalarParameter$ & Comment \\
\midrule
 \text{Normal} & $\frac{1}{\sqrt{(2\pi)\tau^2}}e^{-\frac{1}{2\tau^2}\p*{x-\mu}^2}$ & $\placeholder^2$ & $\sqrt{2\tau}$ & $\standardDeviation \in \realNumbers_+$\\ 
 \text{Laplace} & $\frac{1}{2b}e^{-\frac{\abs{x-\mu}}{b}}$ & $\abs{\placeholder}$ & $b$ & $b \in \realNumbers_+$\\
\bottomrule
\end{tabular}}
\end{table}

Moving forward, we will adopt the following notation convention: $\noiseDistInstance{\zerosDimension}$ will be written as $\noiseDistParametrized$.  When $\mean{} \neq \zerosDimension$, we write  $\noiseDistInstance{\mean{}}$.

\end{noiseDistribution}

The choice of $\noiseDistParametrized{}$ has a direct effect on the geometry of $\constantClassificationSet(\sample)$.  For example, it has been shown that choosing $\noiseDistParametrized{}$ to be the isotropic multivariate normal distribution yields classifiers certifiably robust in $\ell_2$-norm~\cite{DBLP:conf/icml/CohenRK19}.






 \subsection{Randomized Smoothing and Noise Augmentation}\label{subsec:smoothingAndNoiseAugmentation}

We begin by stating the definition of randomized smoothing:

\begin{definition}[Randomized Smoothing]
Given a fixed hypothesis class $\hypothesisClass : \inputSpace \to \outputSpace{}$, a noise distribution $\noiseDistParametrized$, and training data $\trainingData{}$, randomized smoothing yields $\hypContinuous^{\opt}_{\noiseDistParametrized}$:
\begin{align}
    \hypContinuous^{\opt}_{\noiseDistParametrized}\definedAs \arg\max_{\labelElement{} \in \outputSpace}&\quad\Pr{\hyp^{\opt}(\sample + \noiseInstance)=\labelElement{}} \\
    \text{where} &\quad \noiseInstance \drawnFrom \noiseDistParametrized{} \\
    &\quad \hyp^{\opt} = \learningAlgorithm_{\hypothesisClass}(\trainingData{}) \label{eq:targetSmoothingReplacement}
\end{align}
\end{definition}

The smoothed hypothesis, $\hypContinuous^{\opt}_{\noiseDistParametrized}$, results from  performing the smoothing operation  on $\hyp^{\opt}$ where $\hyp^{\opt} \definedAs \learningAlgorithm_{\hypothesisClass,\loss_{\zeroOne}}(\trainingData{})$. We note the difference between our formulation of randomized smoothing and that of Cohen et al \cite{DBLP:conf/icml/CohenRK19}: Whereas they consider randomized smoothing on an arbitrary classifier, we analyze randomized smoothing on the classifier which as $0$ true risk.  As such, our data distribution $\sampleDist{}$, is a distribution over $\inputSpace{} \times \outputSpace{}$.

To precisely understand $\hypContinuous^{\opt}_{\noiseDistParametrized}$, we would need access to $\hyp^*(\sample)$ for all $\sample \in \inputSpace$, but clearly, $\unlabeledTrainingData{}$ is a strict subset of $\inputSpace$.  We further discuss the realizability of classifiers under the randomized smoothing operation as a consequence of the a No-Free-Lunch theorem \cite{shalev2014understanding}. We also discuss the learnability of classifiers when the training process utilizes noise augmentation.  Often, the empirical evaluation of randomized smoothing typically requires a preceding  training procedure~\cite{DBLP:conf/icml/CohenRK19,DBLP:conf/sp/LecuyerAG0J19} with noise augmentation. 
In this paper, we explore, through a statistical-learning theoretic lens, the full randomized smoothing procedure, which includes noise augmentation.





Let us now formally introduce the noise augmented hypothesis:

\begin{definition}[Noise Augmented Hypothesis]\label{def:noiseAugment}
Consider a hypothesis $\hyp \in \hypothesisClassRestricted$, and denote 
\begin{equation}
    \hyp_{\noiseDistParametrized}(\sample) \definedAs \arg\max_{\labelElement{} \in \outputSpace{}} \sum_{\sample' \in \unlabeledTrainingData{}~:~\hyp(\sample')=\labelElement{}} \hyp(\sample')\cdot \noiseDistParametrized{}(\sample-\sample')
\end{equation}
 then $\hyp_{\noiseDistParametrized}(\sample)$ is the noise augmented, by distribution $\noiseDistParametrized$, version of hypothesis $\hyp$ where $\arg\max$ is only defined if there is
a unique maximum.
    
    The noise augmented hypothesis class, $\hypothesisClassSmooth$ is now defined:
    \begin{equation}
        \hypothesisClassSmooth \definedAs \s{\hyp_{\noiseDistParametrized} : \hyp \in \hypothesisClassRestricted}
    \end{equation}
    
\end{definition}

We take sample $\sample_i$ \emph{influences} sample $\sample_j$ to mean $\noiseDistParametrized{}(\sample_j-\sample_i)>0$.


Let us discuss the motivation for our definition of $\hyp_{\noiseDistParametrized}(\sample)$: Suppose the label for any sample in $\inputSpace{}$ is deterministic: For any $(\sample, \labelElement{}) \in \trainingData{}$, we have $\Pr[(\sample,\labelElement{})\drawnFrom \sampleDist]{\outputSpace{} = \labelElement{} \given \inputSpace{} = \sample} = 1$. 
We remark upon the learnability of certain hypothesis classes under randomized smoothing in the context of \cref{thm:impossible}, the No-Free-Lunch Theorem of Shalev-Shwartz and Ben-David \cite{shalev2014understanding}.

\begin{theorem}[No-free-lunch theorem \cite{shalev2014understanding}]\label{thm:impossible}
Let $\learningAlgorithm{}$ be any learning algorithm for the task of binary classification with respect to the $0-1$ loss $\loss_{\zeroOne}$ over a domain $\inputSpace{}$.  Let the number of training samples $\numSamples$ be any number smaller than $\frac{\abs{\inputSpace{}}}{2}$.  Then, there exists a distribution $\sampleDist{}$ over $\inputSpace{}\times\s{0,1}$ such that
\begin{enumerate}
    \item There exists a function $\predictionRule: \inputSpace \to \s{0,1}$ with $\trueRisk(\predictionRule)=0$.
    \item With probability of at least $\sfrac{1}{7}$ over the choice of $\trainingData{}$ sampled from $\sampleDist$, we have that $\trueRisk(\learningAlgorithm(\trainingData{}))\geq \sfrac{1}{8}$
\end{enumerate}
\end{theorem}

We now generalize \cref{thm:impossible} to the multi-class setting in \cref{cor:impossibleMultiClass}.  With \cref{thm:impossible} at our disposal, the corollary is straightforward to show, and we omit its proof.  

\begin{corollary}\label{cor:impossibleMultiClass}
Let $\learningAlgorithm{}$ be any learning algorithm for the task of binary classification with respect to the $0-1$ loss $\loss_{\zeroOne}$ over a domain $\inputSpace{}$.  Let the number of training samples $\numSamples$ be any integer smaller than $\frac{\abs{\inputSpace{}}}{2}$, and let the number of classes $\abs{\outputSpace}$ be any integer larger than $1$.  Then, there exists a distribution $\sampleDist{}$ over $\inputSpace{}\times\outputSpace{}$ such that
\begin{enumerate}
    \item There exists a function $\predictionRule: \inputSpace \to \s{0,1}$ with $\trueRisk(\predictionRule)=0$.
    \item With probability of at least $\sfrac{1}{7}$ over the choice of $\trainingData{}$ sampled from $\sampleDist$, we have that $\trueRisk(\learningAlgorithm(\trainingData{}))\geq \sfrac{1}{8}$
\end{enumerate}
\end{corollary}

Both \cref{thm:impossible} and its multiclass generalization, \cref{cor:impossibleMultiClass}, suggest the difficulty of learning from an unknown distribution $\sampleDist{}$. Given a noise distribution $\noiseDistParametrized{}$, it is easy to construct a sample distribution $\sampleDist{}$, from which a training set $\trainingData{}$ is drawn upon which a learned classifier would have training error $0$, but after applying the randomized smoothing operation, the same classifier would have non-zero training error.  This suggests hypotheses realizable on training data prior to smoothing may not be realizable after randomized smoothing.  

While the randomized smoothing operation does not guarantee realizability of all hypotheses when $\scalarParameter$ is less than a critical threshold $\scalarParameterVar$,  we show in \cref{sec:implications} that a proxy for randomized smoothing does yield realizable hypotheses in exactly the same setting.  

The realizability of this proxy on the training data has a dependency on $\scalarParameter$, and only depends on $\sampleDist{}$ insofar as $\trainingData{}$ is drawn from $\sampleDist{}$. Precisely, the proxy for $\predictionRule$ to which we are referring is a convolution of the noise distribution with the samples we do have access to.  Compared to randomized smoothing, to articulate a decision rule learned from the noise augmented hypothesis class, the learner does not require access to distribution $\sampleDist$.  Therein lies the motivation for the analysis of the noise augmented hypothesis class as a proxy for randomized smoothing.


\subsection{Motivating Example}\label{subsec:motivatingExample}
\newcommand{\demoEntry}[3]{\p[\Big]{\inlineMatrix{#1 \\ #2}, #3}}
In \cref{fig:demo}, we provide a binary classification example in 2-D which motivates this paper. Consider
$$
\trainingData{} = \left\{\begin{array}{c c c}
\demoEntry{0}{\sfrac{1}{2}}{-1}, & \demoEntry{-1}{\sfrac{1}{2}}{+1}, & \demoEntry{\sfrac{1}{4}}{\sfrac{1}{2}}{+1}, \\[0.5em]
\demoEntry{\sfrac{1}{2}}{\sfrac{5}{2}}{+1}, & \demoEntry{-\sfrac{13}{10}}{-\sfrac{17}{10}}{+1}, & \demoEntry{\sfrac{1}{2}}{-1}{+1}, \\[0.5em]
\demoEntry{\sfrac{1}{20}}{\sfrac{5}{2}}{+1}, & \demoEntry{-2}{-\sfrac{17}{10}}{+1}, & \demoEntry{1}{0}{+1}
\end{array}\right\}
$$
where only the first point is labeled with $-1$,
and let $\noiseDistInstance{\mean{}}$ be the uniform distribution: 

\begin{equation}\label{eq:demoNoise}
    \noiseDistInstance{\mean{}}(\placeholder)\definedAs\begin{cases} 
    \frac{1}{\scalarParameter^2\pi} & \norm{\vec{\placeholder}-\mean{}}_2 < \scalarParameter \\
    0 & \text{otherwise}
    \end{cases}
\end{equation}

Statistical dispersion of this distribution is tuned by $\scalarParameter$.  In \cref{fig:demo}, we plot $\trainingData{}$.  Samples labeled with $+1$ are marked with a $+$. Samples labeled with $-1$ are marked with a $-$.  We plot the hypothesis $\hyp^{\opt}_{\noiseDistParametrized{}} \in \hypothesisClassSmooth$ which minimizes empirical risk $\empiricalRisk$.  That is, we are plotting $\arg\min_{\hyp \in \hypothesisClassSmooth}\empiricalRisk(\hyp)$. With color blue, we shade the region of $\realNumbers^2$ in which all points are classified as $+1$.  Color red shades the $-1$ region.  In color off-white, the region in which there exists no influence from any point in $\trainingData{}$.  In color gray, we denote the region in which points are influence by both classes, but the net result is an undefined classification.  

Let us now discuss what occurs when $\scalarParameter$ changes. For $\scalarParameter = \frac{1}{2^3}$, we see the hypothesis class $\hypothesisClass_{\unlabeledTrainingData{},\noiseDist_{\frac{1}{2^3}}}$ can realize $\trainingData{}$.  That is, we can achieve training error $0$.

For $\scalarParameter = \frac{1}{2^1}$, we see the hypothesis class $\hypothesisClass_{\unlabeledTrainingData{},\noiseDist_{\frac{1}{2^1}}}$ can no longer realize $\trainingData{}$. It is apparent the classification of individual samples by $\hyp^{\opt}_{\noiseDist_{\frac{1}{2^1}}}=\arg\min_{\hyp \in \hypothesisClass_{\unlabeledTrainingData{},\noiseDist_{\frac{1}{2^1}}}} \empiricalRisk(\hyp)$ is influenced by the labels of samples near each other.  There exist regions in which points are classified $+1$; however, no samples in $\trainingData{}$ with label $-1$ are classified as $-1$.

For $\scalarParameter = \frac{1}{2^{-1}}$, we see a significant disparity between the the empirical risk of hypotheses $\hyp^{\opt}$ and $\hyp^{\opt}_{\noiseDist_{\frac{1}{2^{-1}}}}=\arg\min_{\hyp \in \hypothesisClass_{\unlabeledTrainingData{},\noiseDist_{\frac{1}{2^{-1}}}}} \empiricalRisk(\hyp)$. There exist regions in which points are classified $+1$; however, no samples in $\trainingData{}$ with label $-1$ are classified as $-1$.

Across all three values of $\scalarParameter$, we notice a trend which we state in informally: When $\abs{\hypothesisClassRestricted}$ decreases, the accuracy associated with the most accurate classifier in $\hypothesisClassRestricted$ also seems to decrease.  

\begin{figure*}[htbp]
\centering
\hspace{-1em}
\subfigure[$\scalarParameter=\frac{1}{2^{3}}$]{
    \label{fig:demo/1}
    \includegraphics[width=0.32\linewidth]{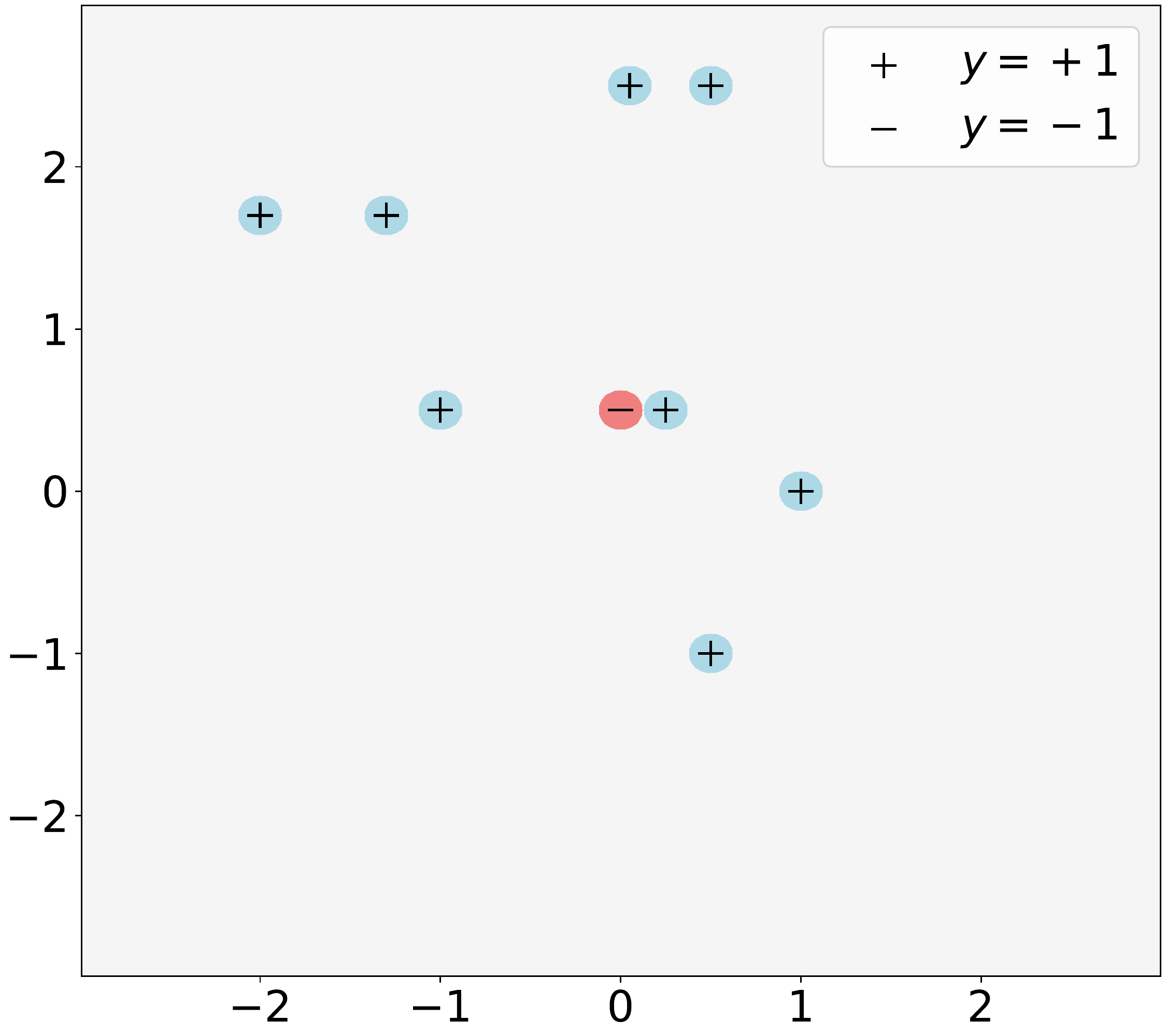}
}
\subfigure[$\scalarParameter=\frac{1}{2^{1}}$]{
    \label{fig:demo/2}
    \includegraphics[width=0.284\linewidth]{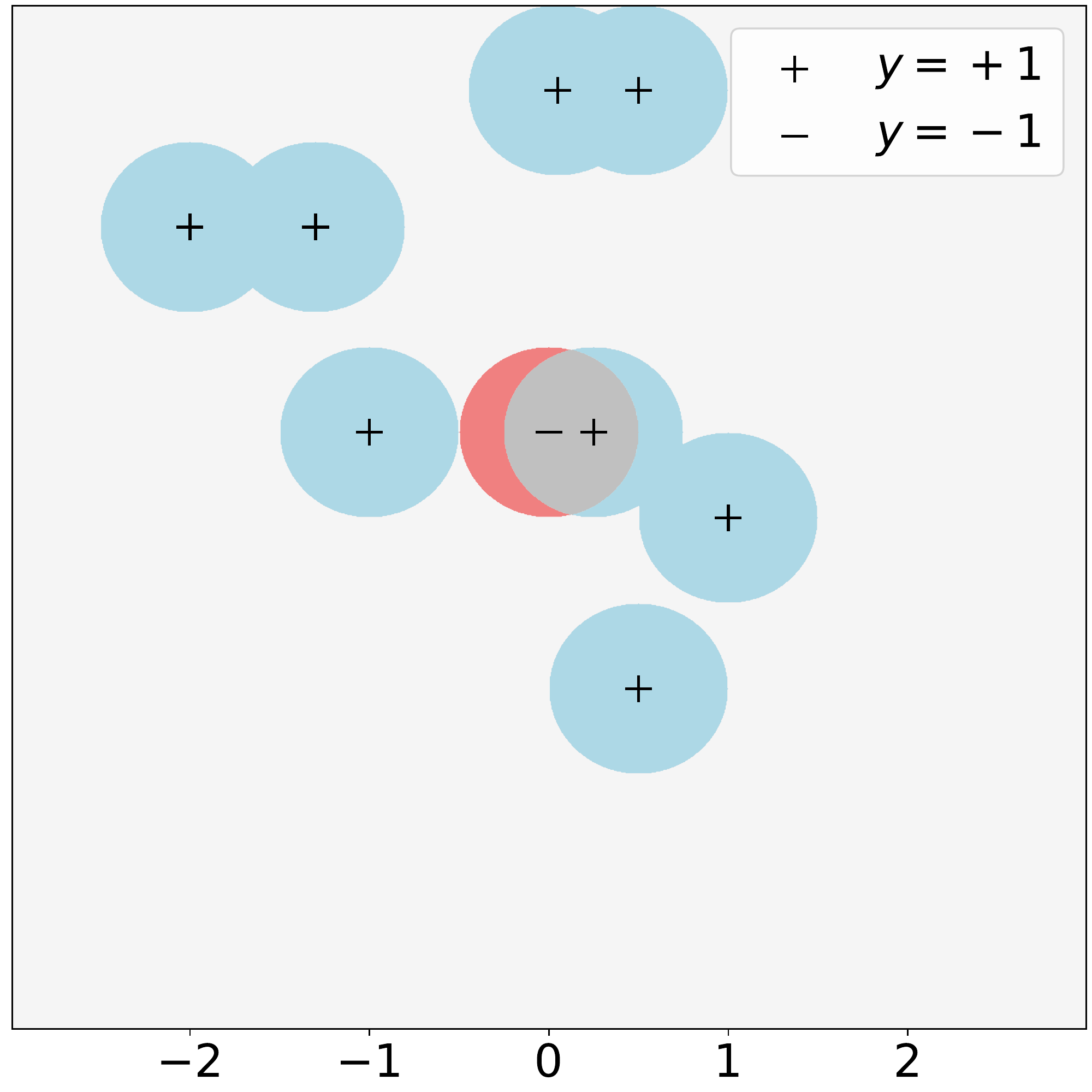}
}
\subfigure[$\scalarParameter=\frac{1}{2^{-1}}$]{
    \label{fig:demo/3}
    \includegraphics[width=0.32\linewidth]{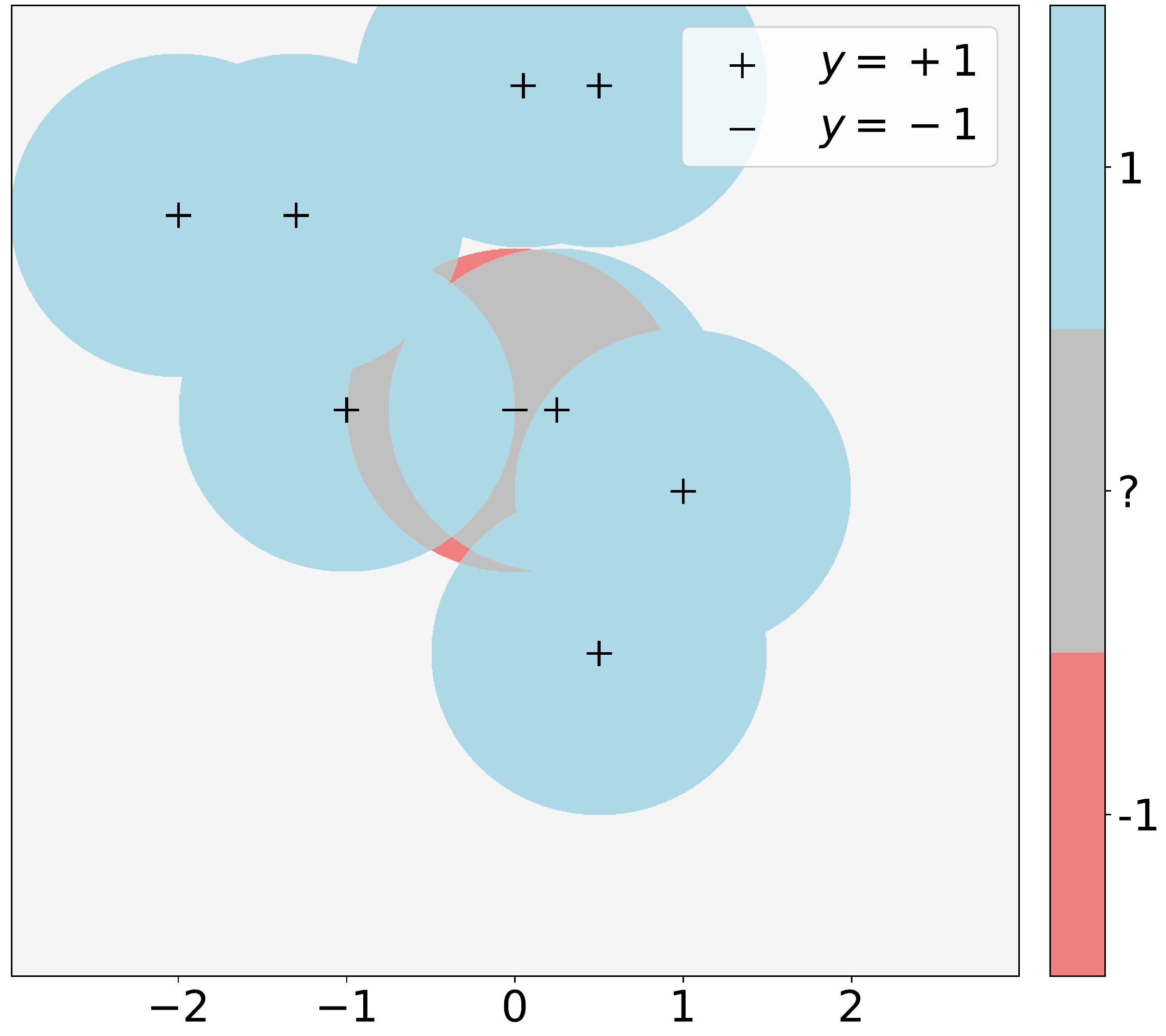}
}
\hspace{-1em}
\caption{For the $\noiseDistInstance{\mean{}}$ specified in \cref{eq:demoNoise}, we plot for $\sgn\s{\sum_{\sample' \in \unlabeledTrainingData{}} \hyp(\sample')(\delta \convolutionOperator \noiseDistInstance{\mean{}})}(\sample-\sample')$ for $\scalarParameter \in \s{\frac{1}{2^{3}},\frac{1}{2^{1}},\frac{1}{2^{-1}}}$.  Notice due to the noise distribution, points are classified (blue for $+1$, red for $-1$, gray for undefined) differently from original label denoted by a marker: $+$ for $+1$, $-$ for $-1$. }
\label{fig:demo}
\end{figure*}

\section{Conditions for Label Clustering}
\label{sec:labelClustering}


This section lays the groundwork necessary for reasoning about how many hypotheses in $\hypothesisClassRestricted$ are realizable after randomized smoothing. In particular, we study  how $\scalarParameter,\noiseDistParametrized$ and $\hypothesisClassSmooth$ relate. We use a counting argument to show that $\hypothesisClassSmooth$ is a subset of $\hypothesisClass_{\unlabeledTrainingData{}}$.  We discuss the implications of this result in \cref{sec:implications}.

Within \cref{sec:labelClustering,sec:implications}, we assume $\hypothesisClass$ shatters $\unlabeledTrainingData{}$ and $\numSamples \geq 3$.  This is not very restrictive: Recent empirical results suggest that even with random label assignment, deep networks can \emph{memorize} their training data. A quote by Zhang et. al.  \cite{random-labeling} put their findings succinctly:

\begin{quote}
    ...when trained on a completely random labeling of the true data, neural networks
achieve 0 training error.
\end{quote}

\subsection{Statistical Dispersion and $\hypothesisClassSmooth$}

We begin with the following theorem, our principal result,  which shows the existence of a threshold $\scalarParameterVar$ which dichotomizes $\scalarParameter$ with respect to $\abs{\hypothesisClassSmooth}$:


\begin{theorem}\label{thm:threshold}
Let $\unlabeledTrainingData{}$ be shattered by $\hypothesisClass$ and let $\abs{\unlabeledTrainingData{}} = \numSamples \geq 3$.  There exists a positive real number $\scalarParameterVar$ such that
\begin{enumerate}[label=\upshape(\roman*), ref=\thetheorem(\roman*)]
\item\label[thmenum]{item:neq}  when $\scalarParameter > \scalarParameterVar$, $\hypothesisClassSmooth \subset \hypothesisClass_{\unlabeledTrainingData{}} $ , and
    \item\label[thmenum]{item:eq} when $\scalarParameter < \scalarParameterVar$, $ \hypothesisClassSmooth= \hypothesisClass_{\unlabeledTrainingData{}}$
    
\end{enumerate}

Furthermore $\scalarParameterVar$ is the solution to 
\begin{equation} \label{eq:solveScalarParameterSmoothUpper}
\begin{split}
     \max \quad& \scalarParameter \\
    \textrm{s.t.} \quad & \s{\noiseDistInstance{\sample}(\sample) \geq \sum_{\sample' \in \s{\unlabeledTrainingData{}\setminus \sample}}\noiseDistInstance{\sample}(\sample')}  \text{ for all $\sample \in \unlabeledTrainingData{}$}
    \end{split}
\end{equation}
\end{theorem}

\begin{proof}

 We will first prove \cref{item:neq}. That is, we will show if  $\scalarParameter < \scalarParameterVar$, then $\abs{\hypothesisClassSmooth} \subset \abs{\hypothesisClass_{\unlabeledTrainingData{}}}$. Because $\unlabeledTrainingData{}$ is shattered by $\hypothesisClass$, it suffices to show when $\scalarParameter > \scalarParameterVar$, $\smoothingFunction$, $\abs{\hypothesisClassSmooth} < \abs{\hypothesisClassRestricted}$.


Assume, without loss of generality, that $-1, +1\in \outputSpace{}$.  Consider training data $\trainingData{\marker}$ in which all labels $\labelElement{i}$ in $\s{\s{\labelElement{1},\hdots,\labelElement{\numSamples}} \setminus \s{\labelElement{\marker}}}$ hold value $+1$, whereas $\labelElement{\marker}$ holds value $-1$.

 Then it is the case $\noiseDistInstance{\sample}(\sample_{\marker}) < \sum_{\sample' \in \s{\unlabeledTrainingData{} \setminus \sample}}\noiseDistInstance{\sample'}(\sample_{\marker})$, it must be the case that $\hyp_{\noiseDistParametrized}(\sample_{\marker}) = +1$. As no $\hyp_{\noiseDistParametrized} \in  \hypothesisClassSmooth$ can realize $\trainingData{\marker}$ we conclude that  when $\scalarParameter > \scalarParameterVar$, $\hypothesisClassSmooth \subset \hypothesisClass_{\unlabeledTrainingData{}} $.

We now prove \cref{item:eq}: That is, we will prove if $\scalarParameter < \scalarParameterVar$, then $\abs{\hypothesisClassSmooth} = \abs{\hypothesisClassRestricted}$:

When $\scalarParameter < \scalarParameterVar$, we have that, for all $\sample \in \unlabeledTrainingData{}$, $\noiseDistInstance{\sample}(\sample) > \sum_{\sample' \in \s{\unlabeledTrainingData{}\setminus \sample}}\noiseDistInstance{\sample'}(\sample)$.

Furthermore, it is not hard to see that for all $\sample \in \unlabeledTrainingData{}$, and for any $\trainingData{} \in  \unlabeledTrainingData{} \times \outputSpaceInstance{}$, the following is true:
\begin{equation}
\noiseDistInstance{\sample}(\sample) \geq \abs{\sum_{(\sample',\labelElement{}') \in \s{\trainingData{}\setminus (\sample,\labelElement{})}}\labelElement{}'\cdot\noiseDistInstance{\sample'}(\sample)}
\end{equation}

Hence, when $\scalarParameter < \scalarParameterVar$, then $\abs{\hypothesisClassSmooth} = \abs{\hypothesisClassRestricted}$.

We also state the impact of randomized smoothing on adversarial training accuracy:

In \cref{thm:threshold}, we examined the impact of randomized smoothing on the set of realizable hypotheses within the natural accuracy setting.  In \cref{thm:thresholdAdversarial}, we conjecture on the on the interplay between adversarial examples and noise augmentation:

\begin{conjecture}\label{thm:thresholdAdversarial}
Given a noise distribution of the form $\noiseDistParametrized{}$, unlabeled training data $\unlabeledTrainingData{}$, and real number $p \geq 1$, there exists an $\eta$ such that for all $\perturbation \in \s{\perturbation : \norm{\perturbation}_p \leq \perturbationNormBound}$, $\hyp_{\noiseDistParametrized{}}(\sample) = \hyp_{\noiseDistParametrized{}}\p*{\sample + \perturbation}$.  Furthermore, $\eta$ depends on $\unlabeledTrainingData{}$, choice of noise distribution $\noiseDistParametrized{}$ and $p$.
\end{conjecture}

\end{proof}

In terms of the example shown in \cref{fig:demo}, \cref{item:eq} captures  \cref{fig:demo/1} ($\scalarParameter = \frac{1}{2^3}$), as no point in $\unlabeledTrainingData{}$ influences any other, hence $\abs{\hypothesisClassRestricted} = \abs{\hypothesisClass_{\unlabeledTrainingData{},\noiseDist_{\frac{1}{2^3}}}}$.  \Cref{item:neq} captures  \cref{fig:demo/2,fig:demo/3}: samples in $\unlabeledTrainingData{}$ influence one another.  In particular,  $\sample_*$ as defined in the proof is the point labeled $-1$.  We see that the number of samples influenced by any sample $\sample$ varies with $\scalarParameter$. 
\section{Implications on randomized smoothing}
\label{sec:implications}

Before discussing the implications of \cref{thm:threshold}, we provide the sample complexity of PAC learning, as given by  Shalev-Shwartz and Ben-David \cite{shalev2014understanding}:

\begin{pac}
Every finite hypothesis class $\hypothesisClass$ is PAC learnable with sample complexity 
\begin{equation}
    \sampleComplexity{\hypothesisClass}(\error,\confidence) \leq \ceil*{\frac{\log(\sfrac{\abs{\hypothesisClass}}{\confidence})}{\error}}
\end{equation}
where $\ceil{\cdot}$ is the ceiling function.
\end{pac}


Given a deterministic labeling function, $\predictionRule: \inputSpace{} \to \s{\pm 1}$, consider the problem of learning, with hypothesis class $\hypothesisClassSmooth$, from the discrete distribution $\queryDist \subseteq \s{\unlabeledTrainingData{}} \times \predictionRule$.  That is, consider the following risk minimization procedure  $\learningAlgorithm_{\hypothesisClassSmooth,\loss_{\zeroOne}}(\trainingData{\markerAlt})$ where $\trainingData{\markerAlt} \subseteq \queryDist$.  For any value of $\scalarParameter$, $\predictionRule$ is always Agnostic PAC learnable, but the same cannot be said for PAC learnability.

When $\scalarParameter < \scalarParameterVar$, $\predictionRule \in \hypothesisClassSmooth$, hence $\hypothesisClassSmooth$ is PAC learnable on distribution $\queryDist$:  The approximation error $\error_{\approximation}$ is $0$, that is to say assuming we picked the best hypothesis $\hyp^{\opt} \in \hypothesisClassSmooth$, then for any test set drawn from $\queryDist$, the test accuracy is $100\%$.

The setting of $\scalarParameter < \scalarParameterVar$ is more intricate.   There are two sub-cases:

\textbf{Sub-case 1}:  $\predictionRule \in \hypothesisClassSmooth$. The learner shows inductive bias by withholding hypotheses in $\hypothesisClassRestricted$ from $\hypothesisClassSmooth$, yet the learner benefits from a better sample complexity upper bound.  The approximation error $\error_{\approximation}$ remains $0$. That is to say assuming we picked the best hypothesis $\hyp^{\opt} \in \hypothesisClassSmooth$, then for any test set drawn from $\queryDist$, the test accuracy is $100\%$, and, we might reach $\hyp^{\opt}$ with fewer training samples than we needed in the case $\scalarParameter < \scalarParameterVar$.  In this case, $\hypothesisClassSmooth$ is PAC learnable on distribution $\queryDist$.

\textbf{Sub-case 2}: $\predictionRule \not\in \hypothesisClassSmooth$. Here $\predictionRule$ is no longer PAC learnable. Here, the inductive bias induced by withholding hypotheses in $\hypothesisClassRestricted$ from $\hypothesisClassSmooth$ causes the approximation error $\error_{\approximation}$ to be positive: assuming we picked the best hypothesis $\hyp^{\opt} \in \hypothesisClassSmooth$, then there exists a test set drawn from $\queryDist$, the test accuracy is not $100\%$.  In this case, $\hypothesisClassSmooth$ is not PAC learnable on distribution $\queryDist$.

A test set within the setting of sub-case 2 is easy to construct: set $\trainingData{\markerAlt}$ to be $\trainingData{\marker}$ as defined in the proof of \cref{thm:threshold}.

If a learner does not have have a priori knowledge of the prediction rule $\predictionRule$, by smoothing, the learner runs the risk of losing the ability to realize $\predictionRule$.  When $\scalarParameter > \scalarParameterVar$, the learner gains robustness, but is putting up accuracy as collateral.

Making an analogy: If setting $\scalarParameter < \scalarParameterVar$ is buying a treasury bond, then $\scalarParameter > \scalarParameterVar$ is buying penny stocks.  One is low-risk, low-reward, the other is high-risk high-reward, but instead of money, the learner is risking the ability to guarantee $100\%$ test-accuracy for both robustness and a training procedure with fewer samples.

\section{Evaluation}\label{sec:eval}

We conduct extensive experiments on noise augmentation and randomized smoothing to support our conclusions from \cref{sec:implications}. 
In particular, we design the experiments to answer three research questions:

\begin{itemize}[leftmargin=*]
\item Does training with noise augmentation affect the performance of randomized smoothing in terms of both natural and adversarial accuracy?
\item When the Gaussian distribution is used as the noise distribution $\noiseDistParametrized{}$ in both noise augmentation and randomized smoothing, how does varying the scaling parameter, $\scalarParameter$, affect performance of the learned classifier?
\item When the Gaussian distribution is used as the noise distribution $\noiseDistParametrized{}$, how does varying $\scalarParameter$ affect $\hypothesisClassSmooth$?  Our principal interest lies in examining $\abs{\hypothesisClassSmooth}$.
\end{itemize}


%


Previous approaches have studied randomized smoothing in terms of certified accuracy, natural accuracy, and robustness guarantees~\cite{DBLP:conf/icml/CohenRK19}. In this work, we build on previous results and take the novel approach of systematically studying the impact of randomized smoothing and noise augmentation on generalization and adversarial accuracy. 




\subsection{Experimental Highlights}\label{sec:exp/answers}

In this section, we will provide a brief overview to the answers of the above three questions and provide insight to the underlying experiments. The evaluation is conducted on four datasets: MNIST, CIFAR-10, GTSRB, and ImageNet.  We describe these datasets in \cref{sec:exp/setup/dataset,sec:exp/setup/network}.

\textbf{1. Does training with noise augmentation affect the performance of randomized smoothing in terms of both natural and adversarial accuracy?}
To answer this question, for each dataset, we apply randomized smoothing to two classifiers of the same architecture. One classifier is trained \emph{without} noise augmentation.  The other is trained \emph{with} noise augmentation.  For each dataset, we compare the performance of the two classifiers.
Then, we consider the models that are trained \emph{with} noise augmentation but \emph{not} smoothed. Empirical results from \cref{sec:exp/aug} highlight that noise augmentation improve both the accuracy and robustness of the smoothed classifier. These results support our argument that noise augmentation may yield realizable hypotheses which are not realizable under smoothing.  
%
%

\textbf{2. When the Gaussian distribution is used as the noise distribution in both noise augmentation and randomized smoothing, how does varying the scaling parameter, $\scalarParameter$, affect performance of the learned classifier?}
To answer this question, we evaluate the natural and adversarial accuracy of smoothed models trained with noise augmentation. 
The results in \cref{sec:exp/detail/eval} suggests that, while smoothing comes with a cost of the natural accuracy, this accuracy on the test set becomes \emph{closer} to that on the training set as the noise level increases. As for adversarial accuracy, randomized smoothing effectively provides adversarial robustness, but the accuracy seems to decline as the noise level increases beyond a critical threshold.

\textbf{3. When the Gaussian distribution is used as the noise distribution $\noiseDistParametrized{}$, how does varying $\scalarParameter$ affect $\hypothesisClassSmooth$? }
Motivated by the work of Zhang et al.~\cite{random-labeling}, we take a real dataset, and re-assign completely random labels for the samples in the train and test sets. This experiment helps in understanding the impact of noise augmentation and randomized smoothing regardless of the underlying noise distribution. As the random labeling procedure effectively mimics hypothesis shattering on the training set. By overfitting a classifier to the training set, we can identify a realizable hypothesis. We then check if it is realizable after smoothing by checking its test set accuracy.

\subsection{Experimental Setup}\label{sec:exp/setup}

\subsubsection{Datasets}\label{sec:exp/setup/dataset}
We consider four image classification datasets: MNIST~\cite{mnist}, CIFAR-10~\cite{cifar10}, GTSRB~\cite{GTSRB} and ImageNet~\cite{imagenet}.
The MNIST dataset contains $60\mathpunct{}000$ training and $10\mathpunct{}000$ test images of handwritten digits, split into $10$ classes. Each image is of size $28\times28$, with a single color channel.
The CIFAR-10 dataset contains $50\mathpunct{}000$ training and $10\mathpunct{}000$ test images, split into $10$ classes.  Each image is of size $32\times32$, with $3$ color channels.
 GTSRB is a dataset containing German traffic signs.  Each image within GTSRB has $3$ channels and is resized to $32\times32$.  The dataset contains $39\mathpunct{}209$ training examples and $12\mathpunct{}630$ test examples, split into $43$ classes. 
The ImageNet dataset contains $1\mathpunct{}232\mathpunct{}167$ training and $49\mathpunct{}000$ validation images, split into $1\mathpunct{}000$ classes. Each image is resized to $256\times256$.

For each dataset, we normalize each image $\boldsymbol{x}$ by subtracting the channel-wise mean $\mu$ and dividing by the channel-wise standard deviation $\sigma$. The values of $\mu$ and $\sigma$ are computed on the entire training set of each dataset, as given in \cref{tab:setup/dataset/normalize}.

\begin{table}[tb]
\caption{Channel-wise Means and Standard Deviations}
\vspace{0.5em}
\label{tab:setup/dataset/normalize}
\centering
\resizebox{0.6\linewidth}{!}{\begin{tabular}{ccc}
\toprule
Dataset & \multicolumn{1}{c}{Mean ($\mu$)} & \multicolumn{1}{c}{Standard Deviation ($\sigma$)} \\
\midrule
MNIST & $(0.1307)$ & $(0.3081)$ \\
CIFAR-10 & $(0.4914, 0.4822, 0.4465)$ & $(0.2023, 0.1994, 0.2010)$ \\
GTSRB & $(0.3787, 0.3482, 0.3571)$ & $(0.3005, 0.2944, 0.3008)$ \\
ImageNet & $(0.4850, 0.4560, 0.4060)$ & $(0.2290, 0.2240, 0.2250)$ \\
\bottomrule
\end{tabular}}
\end{table}

\subsubsection{Network and Learning Algorithms}\label{sec:exp/setup/network}
We apply different model architectures to each dataset. A small depth $4$ CNN network, a ResNet~\cite{resnet} with depth $110$, a ResNet with depth $20$, and a ResNet with depth $50$ are used for MNIST, CIFAR-10, GTSRB, and ImageNet, respectively. Each model is trained by Stochastic Gradient Descent (SGD) with a momentum parameter of $0.9$. An initial learning rate of $0.01$ (for small CNN) or $0.1$ (for ResNet) are used, with a decay factor of $0.7$ (for MNIST) or $0.9$ (for GTSRB) per training epoch. As for CIFAR-10, we use the same network and settings specified in~\cite{DBLP:conf/icml/CohenRK19}, where the initial learning rate ($0.1$) is decreased by a factor of $0.1$ per $30$ training epochs. For ImageNet, we use the pre-trained models of Cohen et al.~\cite{DBLP:conf/icml/CohenRK19}.

\subsubsection{Randomized Smoothing}\label{sec:exp/setup/smoothing}
As we consider $\ell_2$-bounded adversaries in this evaluation, our choice of noise distribution $\noiseDistParametrized$ is the Normal Distribution from Cohen et al.~\cite{DBLP:conf/icml/CohenRK19}. In \Cref{tab:noiseDistributionExamples}, we showed how to convert the standard notation for the Normal Distribution into that of our generalized noise distribution. We refer to the standard deviation $\sigma$ of the Normal distribution as \emph{noise level}. 

Therefore, we largely reuse the randomized smoothing framework implemented by Cohen et al.~\cite{DBLP:conf/icml/CohenRK19} on Gaussian noise. We randomly draw $N=10\mathpunct{}000$ samples around each input image to predict its label. Other parameters are left as default values, e.g., a failure probability of $0.001$ is used for the Monte Carlo algorithm, which is utilized by randomized smoothing as an internal procedure.

\newcommand{\advExample}{\hat{\boldsymbol{x}}}
\subsubsection{Adversarial Accuracy}\label{sec:exp/definition}
We define adversarial accuracy as the accuracy under the PGD adversarial attack scheme $\mathrm{SmoothAdv}$ proposed by Salman et al.~\cite{smoothing-pgd}. This scheme is designed specifically for randomized smoothing, and is different than the normal PGD attack proposed by Madry et al.~\cite{DBLP:conf/iclr/MadryMSTV18}. Salman et al.~\cite{smoothing-pgd} report that this attack scheme finds adversarial examples which decrease the probability of correct classification of the smoothed classifier. This attack uses $k$ steps to obtain an adversarial example within an $\ell_2$ norm ball of radius $\epsilon$ centered at $\boldsymbol{x}$.

For all experiments, we apply this PGD attack with 20 steps. For each dataset, we consider four values of $\epsilon$. For CIFAR-10, GTSRB and ImageNet, we use $0.25$, $0.50$, $1.00$, and $2.00$. For MNIST, we use $1.50$, $2.00$, $2.50$, and $5.00$. We do not report other values of $\epsilon$, because the resulting observations are similar. Larger values for MNIST were selected because its one-channel images generally require stronger attacks.

In a similar fashion to natural accuracy, adversarial accuracy can be defined as $\frac{1}{N}\sum_{i=1}^N\loss_{\zeroOne}\p[\big]{f\p{\hat{\boldsymbol{x}}_i},y_i}$, where $\hat{\boldsymbol{x}}_i$ is the $i^{\textrm{th}}$ adversarial example obtained by $\mathrm{SmoothAdv}$.

\subsubsection{Random Labeling}\label{sec:exp/setup/random}
When examining random labeling, we focus on the MNIST dataset and using the network and hyperparameters defined in \cref{sec:exp/setup/network}. To allow for an easier overfitting, the training procedure does not utilize DropOut -- as suggested by Zhang et al.~\cite{random-labeling}. For each trial, we generate a new dataset by randomly re-assigning the labels for each image in both the training and test sets. 


\subsection{Q1. Dependence on Noise Augmentation}\label{sec:exp/aug}

Existing works for randomized smoothing require noise augmentation to be used in conjunction with smoothing~\cite{DBLP:conf/icml/CohenRK19, DBLP:conf/sp/LecuyerAG0J19}. In this section, we draw attention to the connection between these two components and its impact on the classification performance. First, we compare the performance of a smoothed model with and without noise augmentation during training. We characterize the performance in terms of natural accuracy and adversarial accuracy.
%

%
%
 We evaluate the natural accuracy of the training and test sets, and the adversarial accuracy  on the test set to demonstrate the effectiveness of adversarial robustness provided by randomized smoothing against a norm-bounded adversary. We then repeat this evaluation for the noise levels and datasets discussed in \cref{fig:exp/aug}. 
 

Several observations can be made from \cref{fig:exp/aug}, but in this section, we first focus on the difference between the performance of models that are trained with and without noise augmentation, i.e., the difference between the \emph{solid} and \emph{dotted} lines.
First, we notice that at zero noise level, the training, natural, and adversarial accuracy are the same for the models with and without noise augmentation, because noise augmentation and smoothing reduce to the normal case. \textbf{As the noise level used in smoothing is increased, we empirically observe that models trained without noise augmentation incur a severe drop in accuracy at a lower noise level $\scalarParameter$ than the noise level at which a similar severe drop in accuracy occurs in models trained with noise augmentation.}
A more insightful observation is that training without noise augmentation reduces the natural accuracy as smoothing begins to provide desirable adversarial accuracy.

%

\begin{figure}[b]
\centering
\subfigure[MNIST]{
    \label{fig:exp/aug/mnist}
    \includegraphics[width=0.23\linewidth]{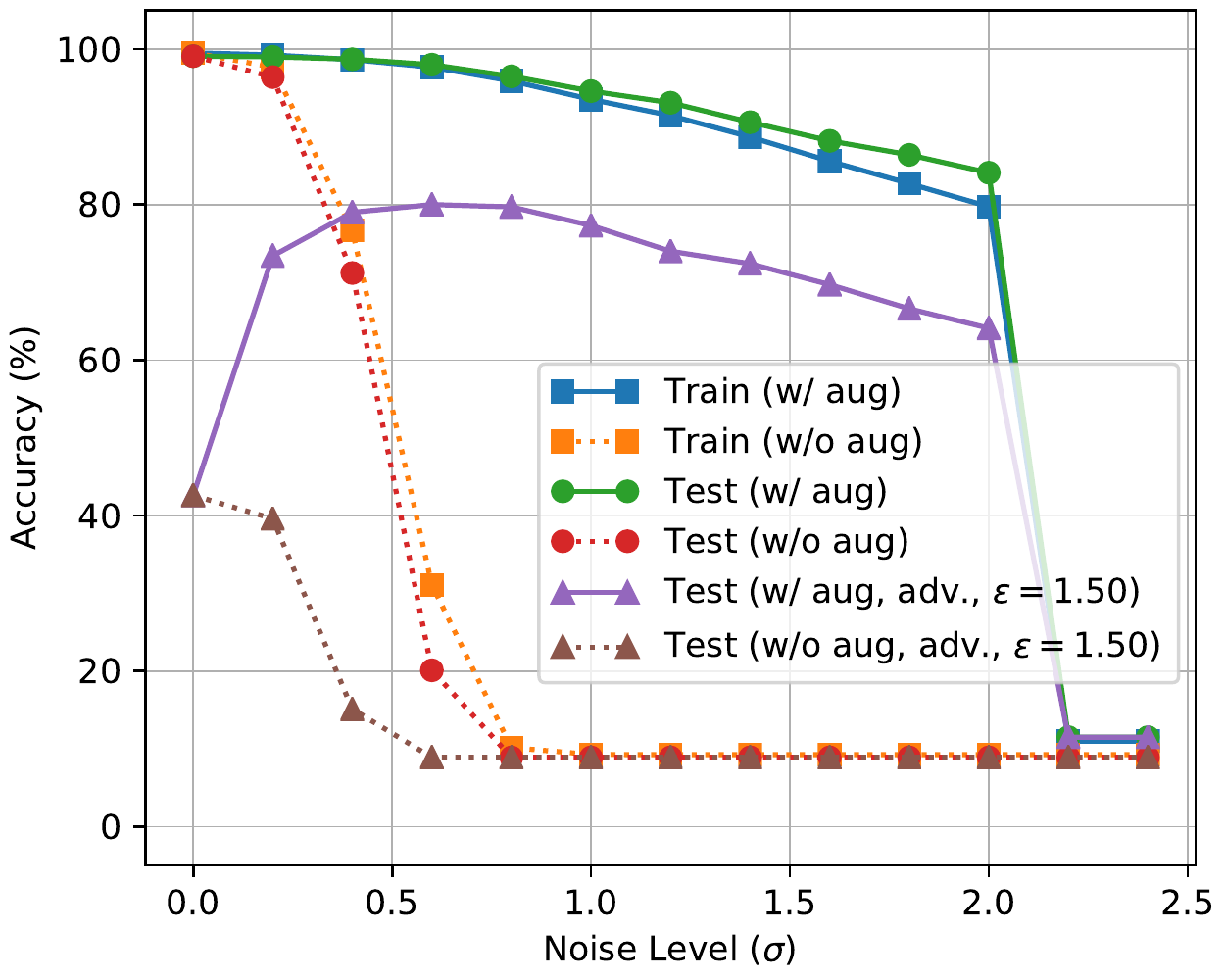}
}
\subfigure[CIFAR-10]{
    \label{fig:exp/aug/cifar}
    \includegraphics[width=0.23\linewidth]{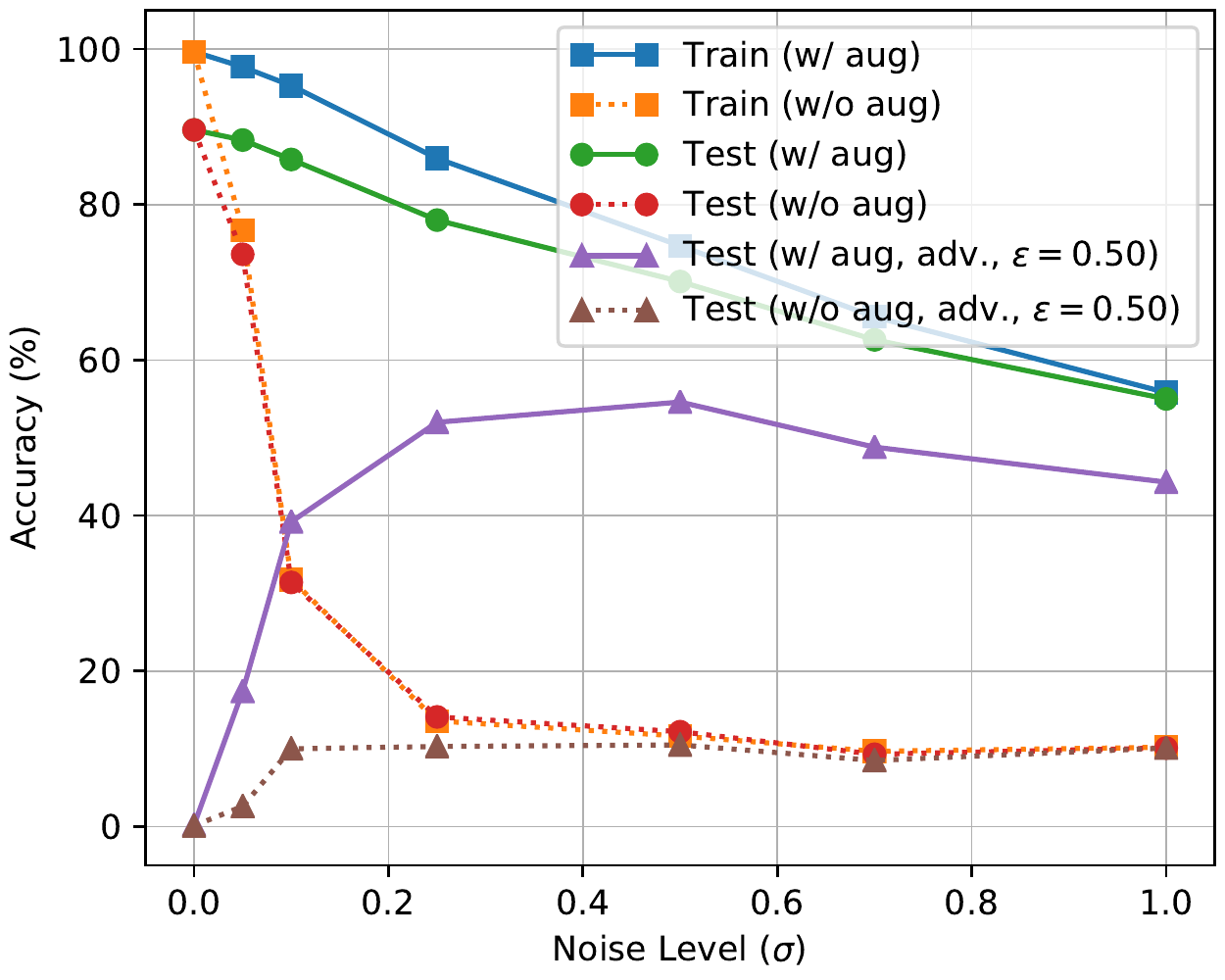}
}
\subfigure[GTSRB]{
    \label{fig:exp/aug/gtsrb}
    \includegraphics[width=0.23\linewidth]{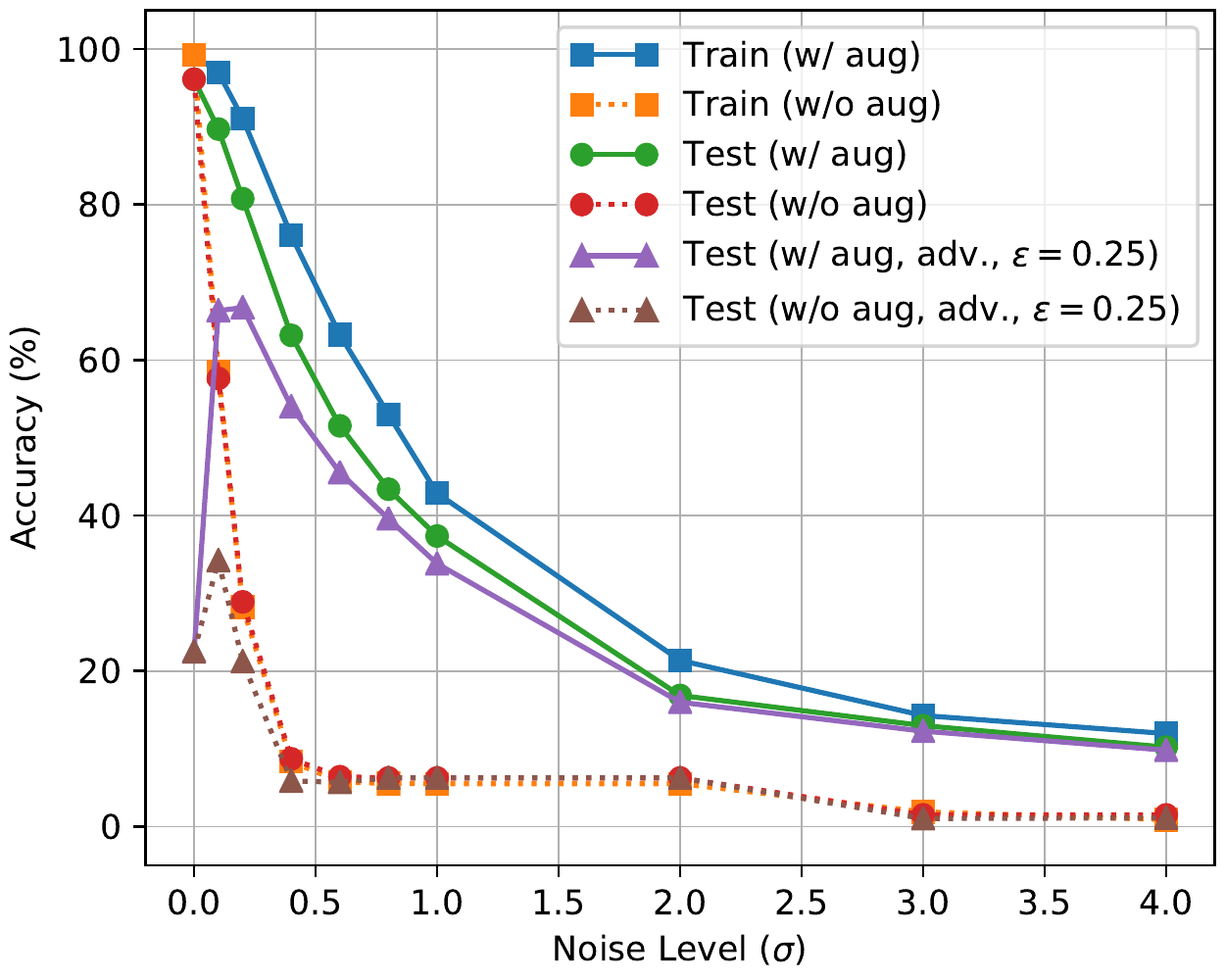}
}
\subfigure[ImageNet]{
    \label{fig:exp/aug/imagenet}
    \includegraphics[width=0.23\linewidth]{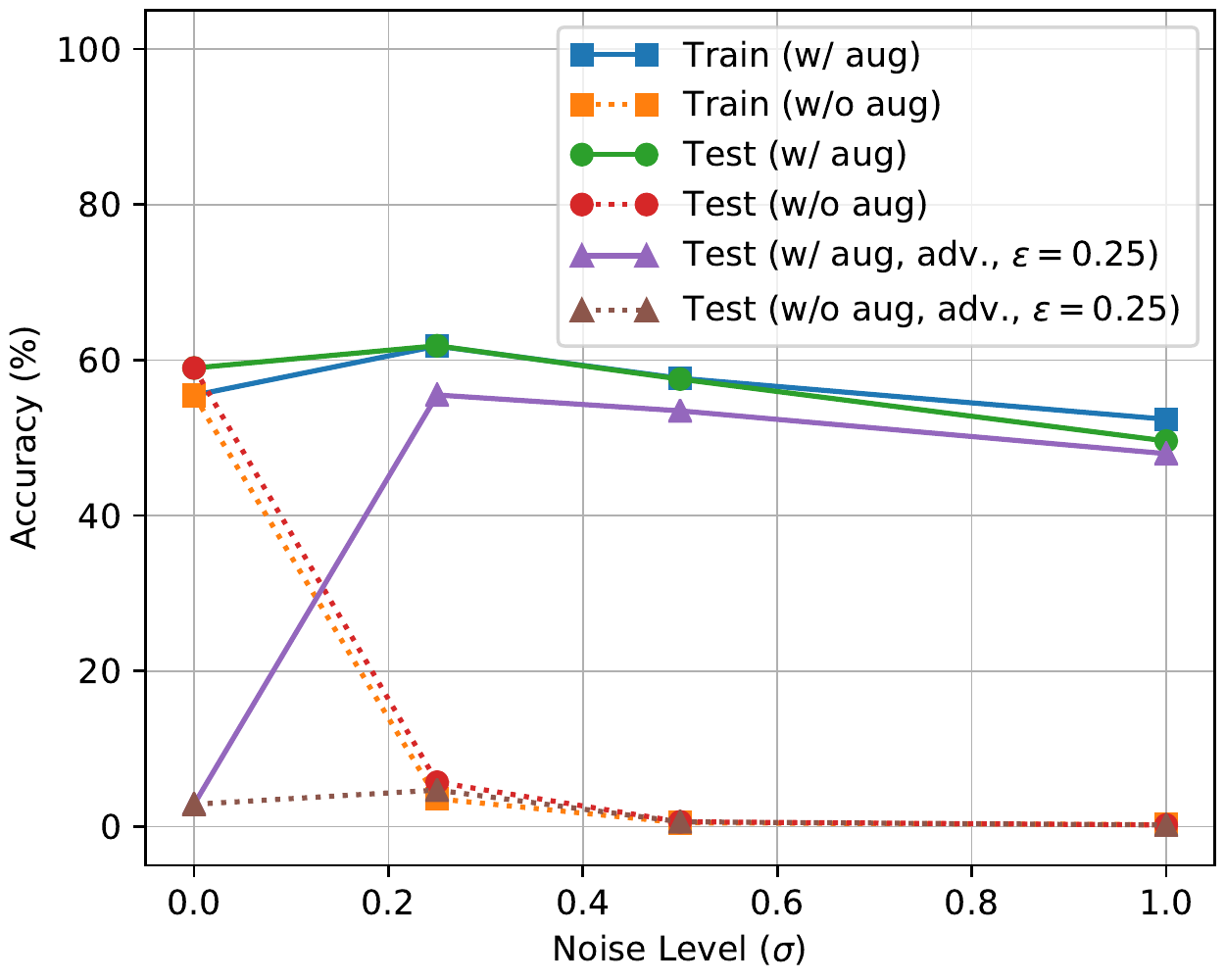}
}
\hspace{-1em}
\caption{Natural and Adversarial Accuracy (\%) vs Noise Levels ($\sigma$) for Randomized Smoothing with/without Noise Augmentation. Triangular markers denote the adversarial accuracy under PGD attack with the maximum $\ell_2$ norm of $\epsilon$. We report only one $\epsilon$ for each dataset to reduce the complexity of this figure. Dotted lines refer to the accuracy of models trained without noise augmentation.}
\label{fig:exp/aug}
\end{figure}


\begin{figure*}[ht]
\centering
\subfigure[MNIST (natural)]{
    \label{fig:exp/rand/mnist/natural}
    \includegraphics[width=0.24\linewidth]{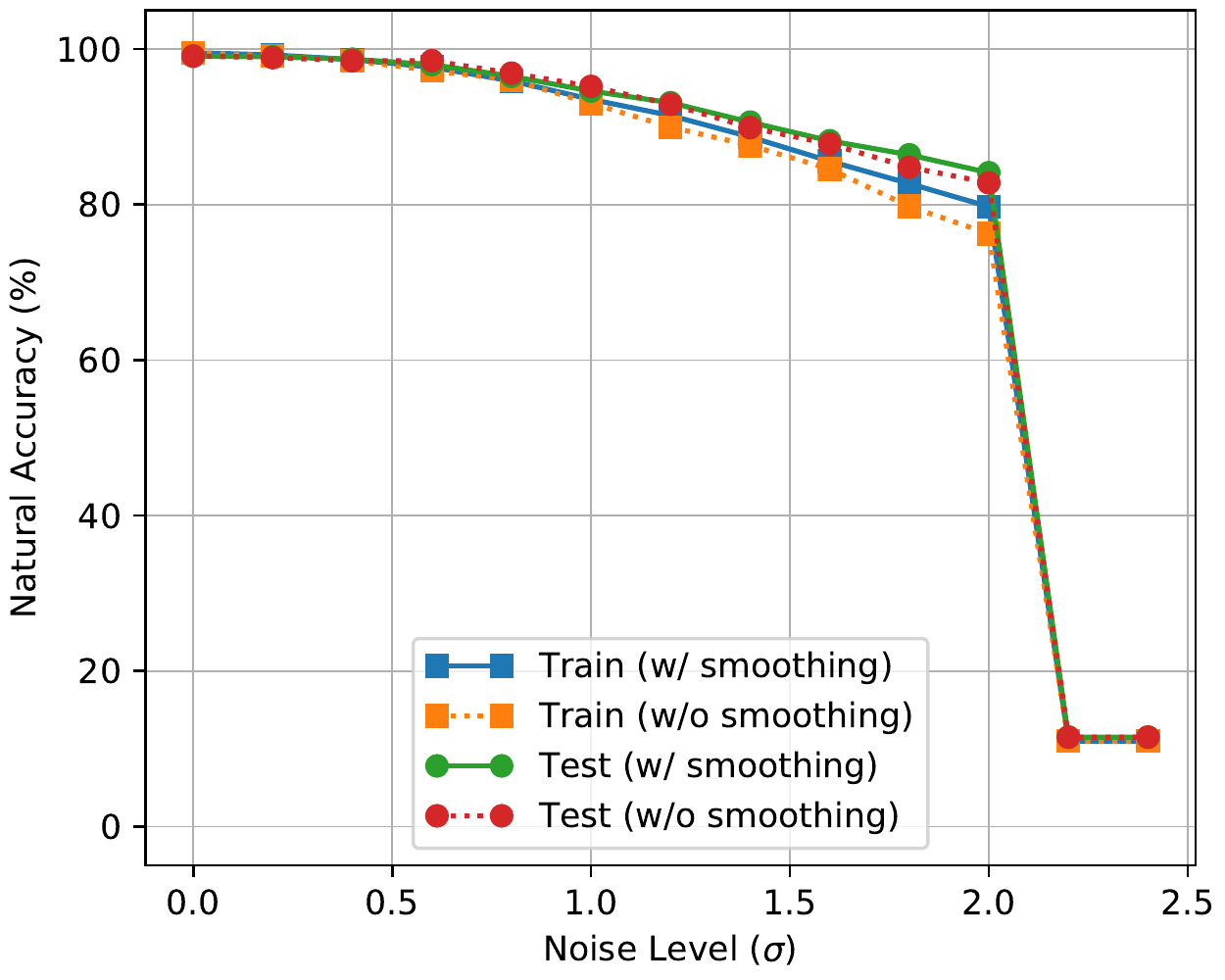}
}
\hspace{-1em}
\subfigure[CIFAR-10 (natural)]{
    \label{fig:exp/rand/cifar/natural}
    \includegraphics[width=0.24\linewidth]{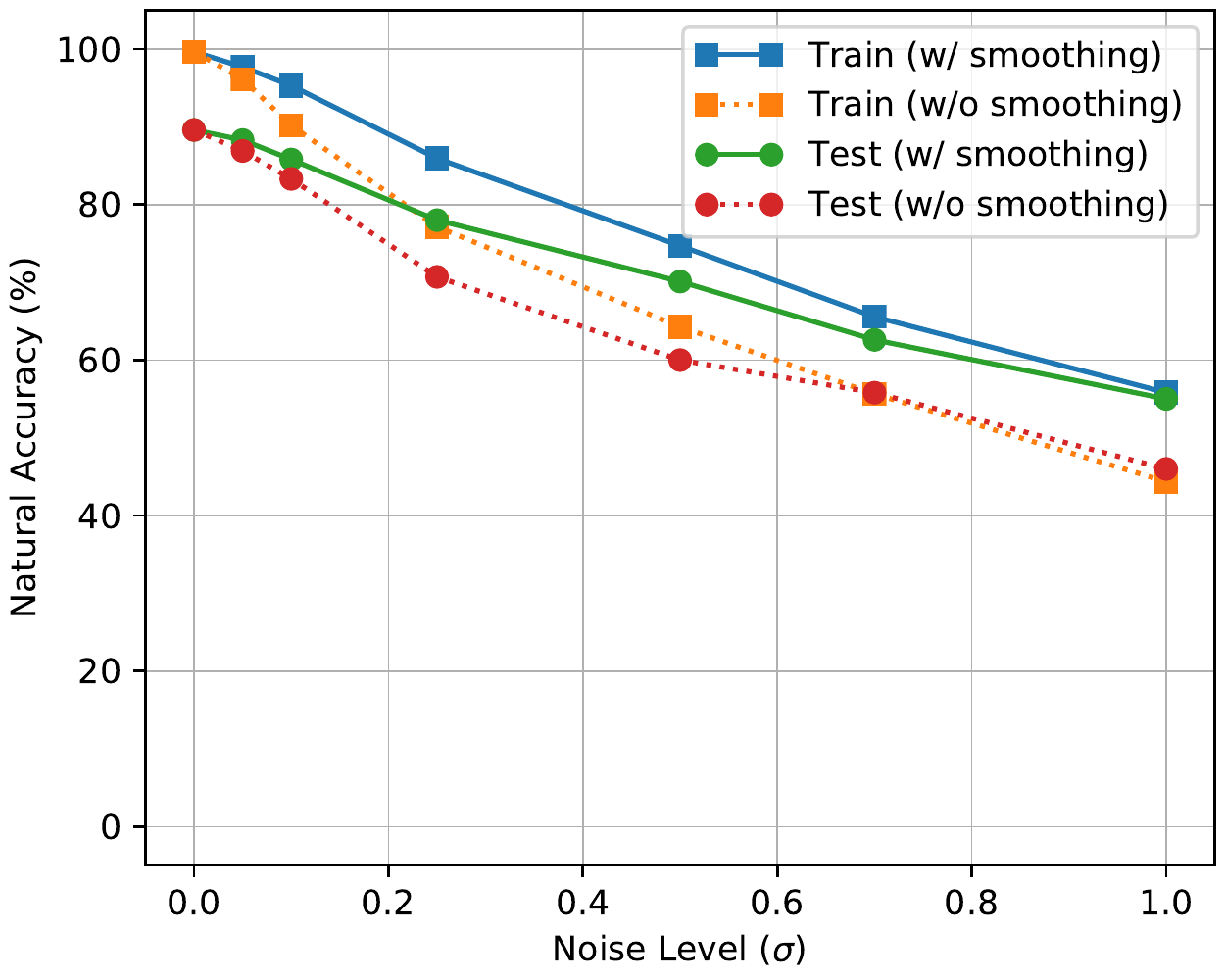}
}
\hspace{-1em}
\subfigure[GTSRB (natural)]{
    \label{fig:exp/rand/gtsrb/natural}
    \includegraphics[width=0.24\linewidth]{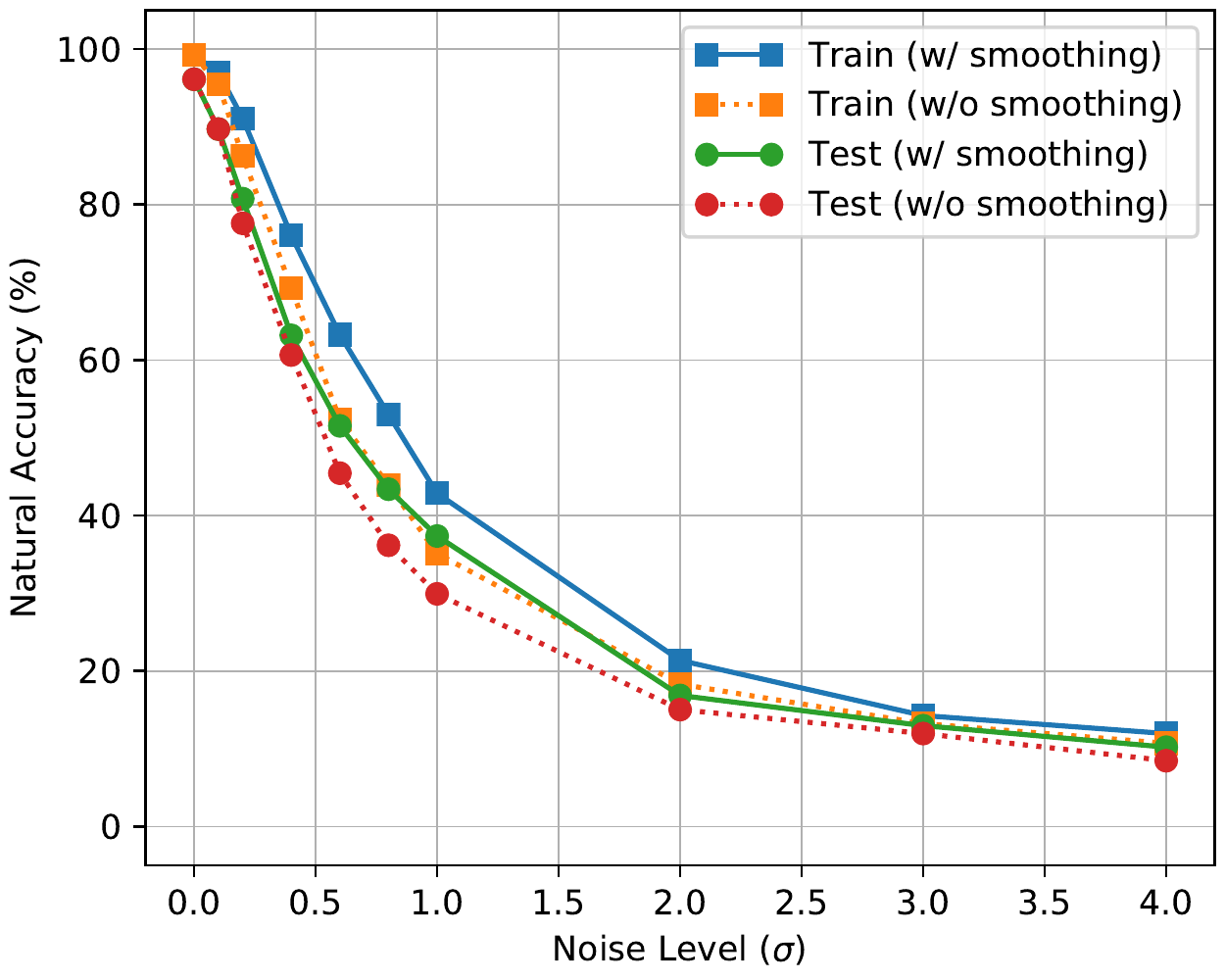}
}
\hspace{-1em}
\subfigure[ImageNet (natural)]{
    \label{fig:exp/rand/imagenet/natural}
    \includegraphics[width=0.24\linewidth]{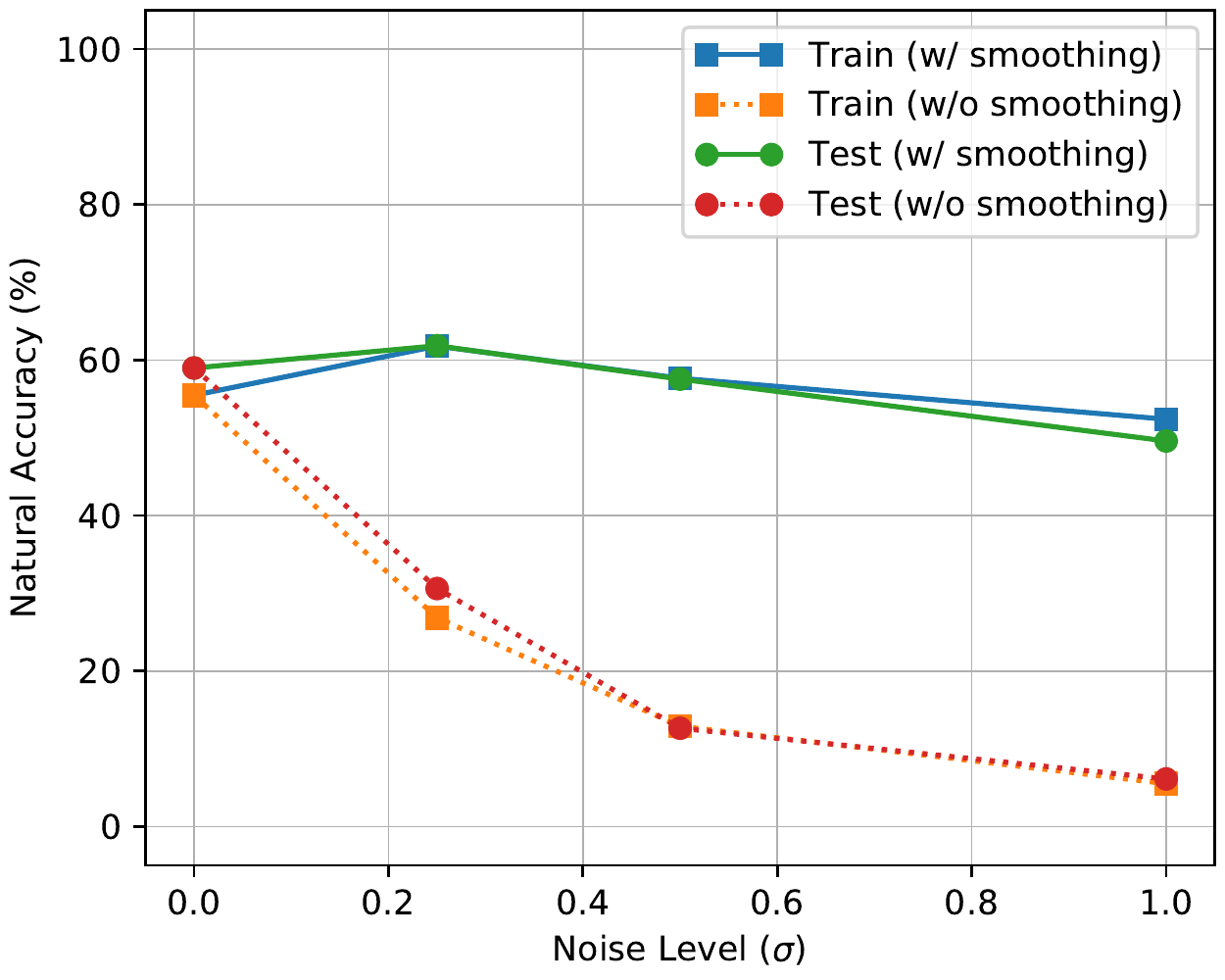}
}
\subfigure[MNIST (adversarial)]{
    \label{fig:exp/rand/mnist/adversarial}
    \includegraphics[width=0.24\linewidth]{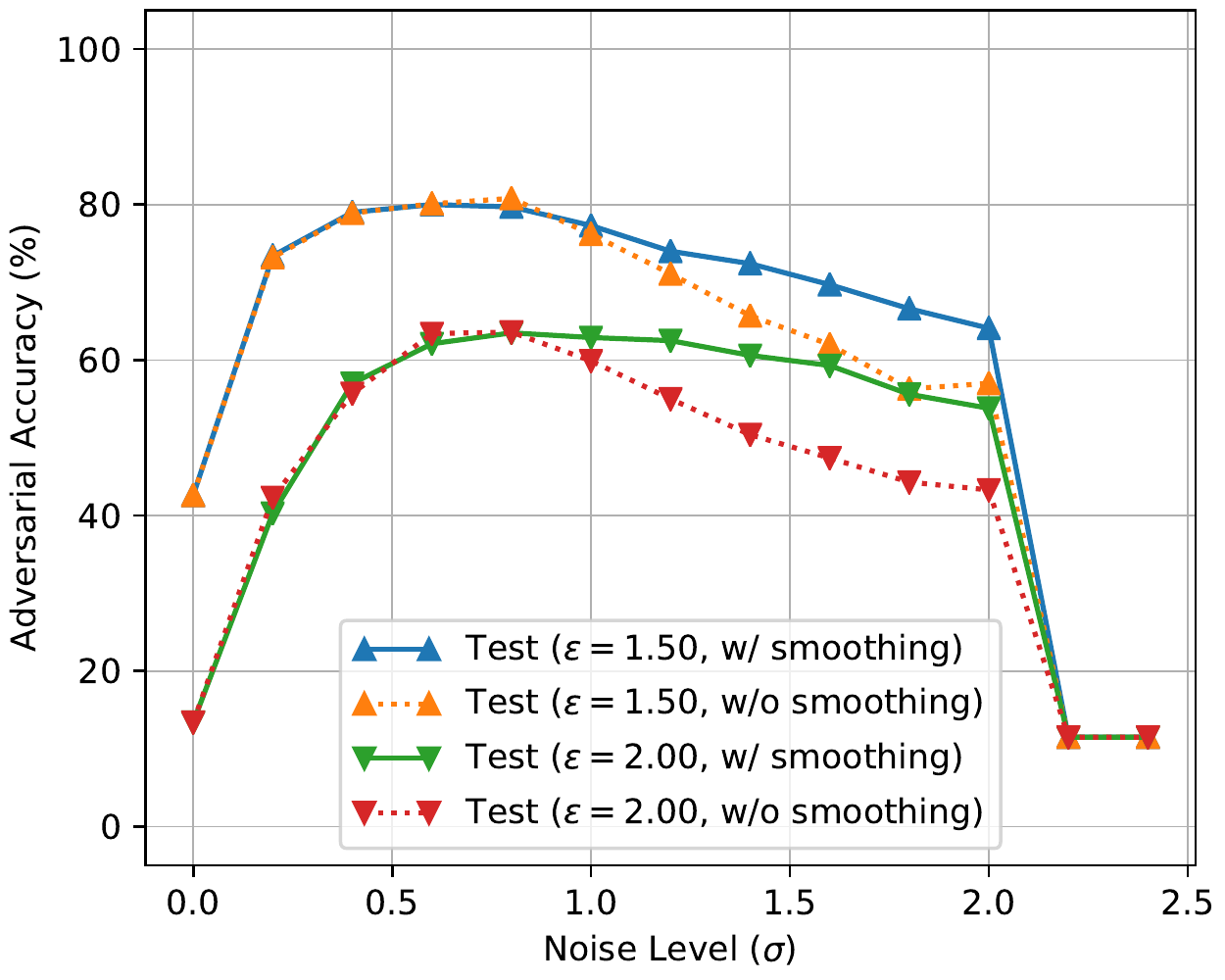}
}
\hspace{-1em}
\subfigure[CIFAR-10 (adversarial)]{
    \label{fig:exp/rand/cifar/adversarial}
    \includegraphics[width=0.24\linewidth]{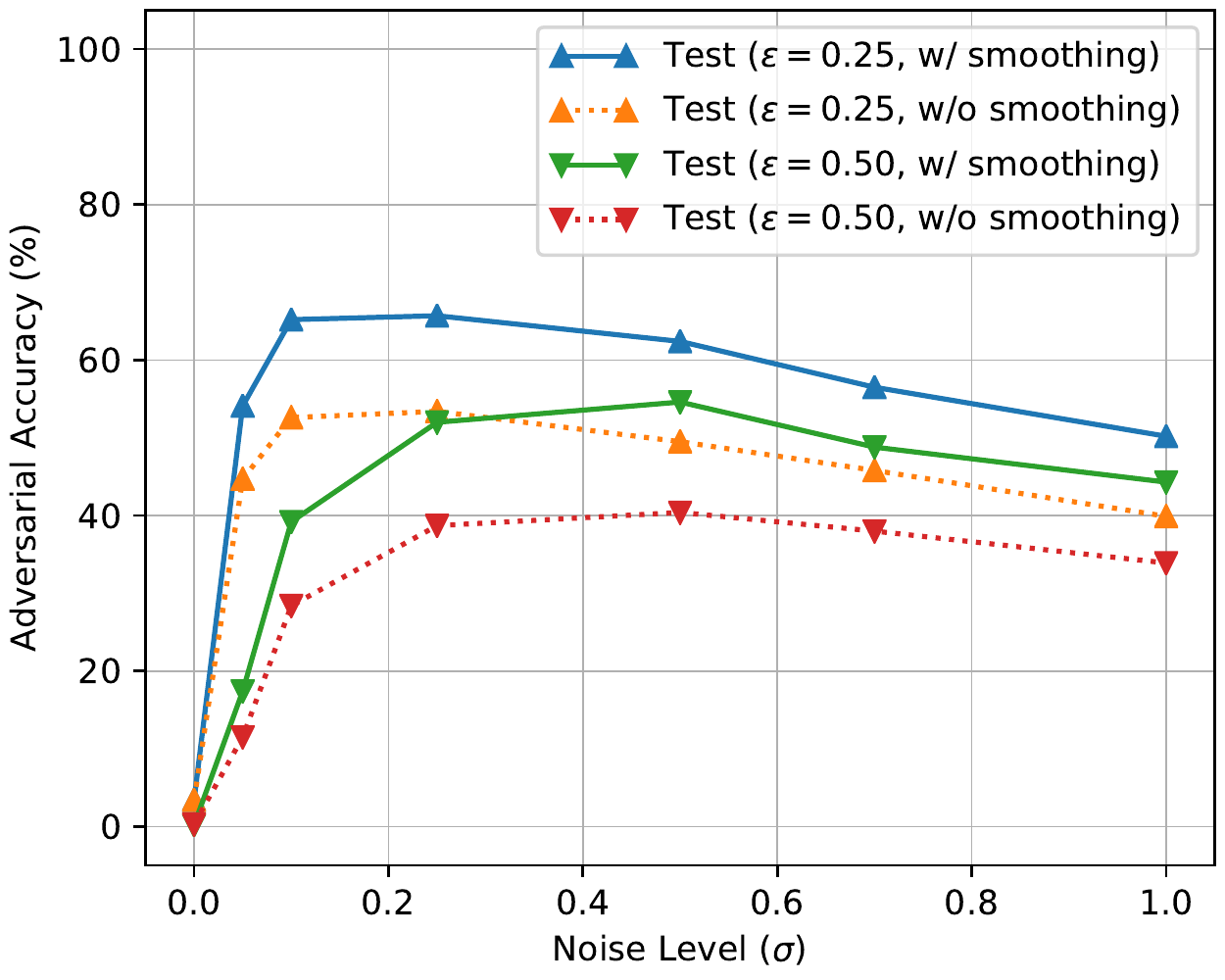}
}
\hspace{-1em}
\subfigure[GTSRB (adversarial)]{
    \label{fig:exp/rand/gtsrb/adversarial}
    \includegraphics[width=0.24\linewidth]{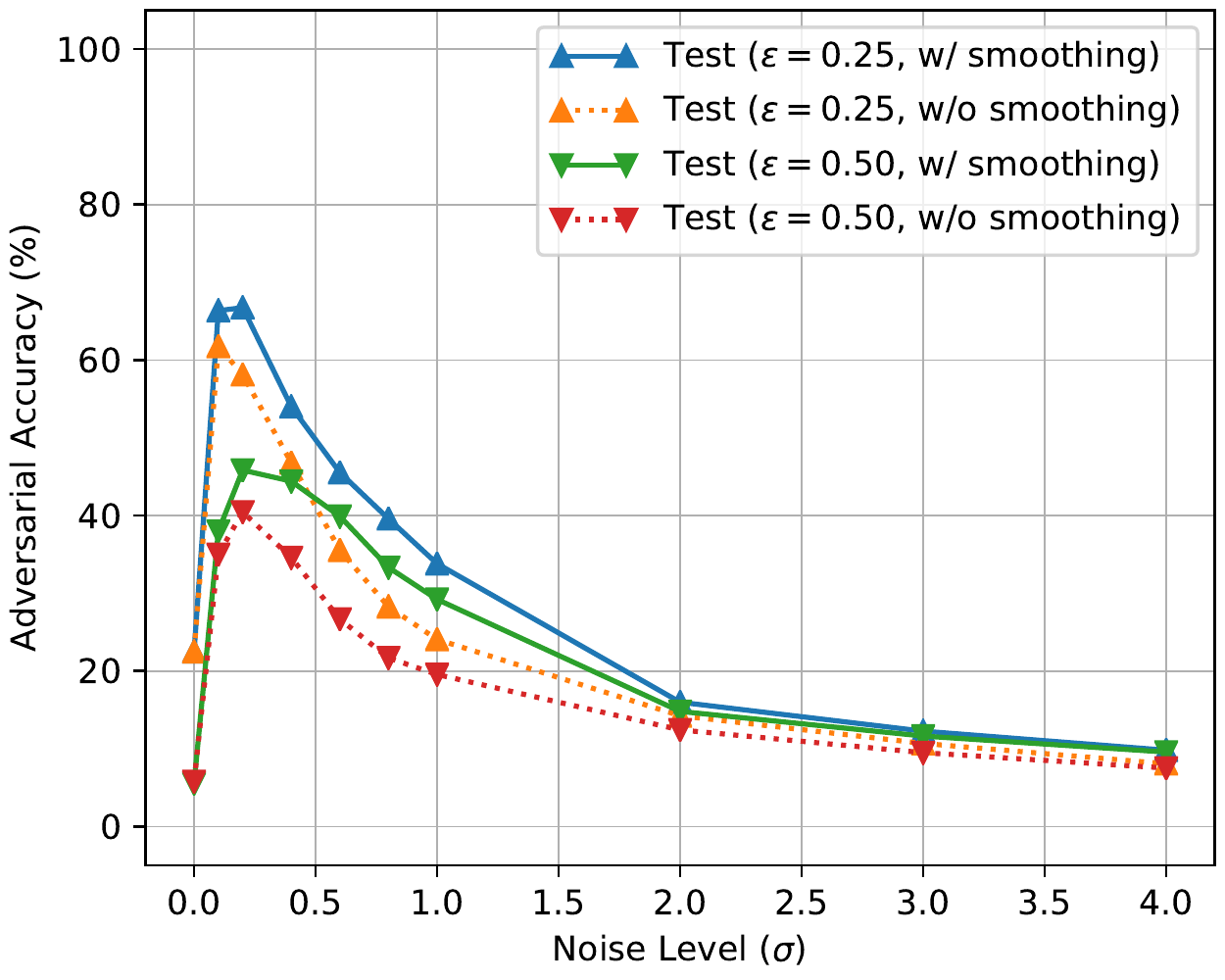}
}
\hspace{-1em}
\subfigure[ImageNet (adversarial)]{
    \label{fig:exp/rand/imagenet/adversarial}
    \includegraphics[width=0.24\linewidth]{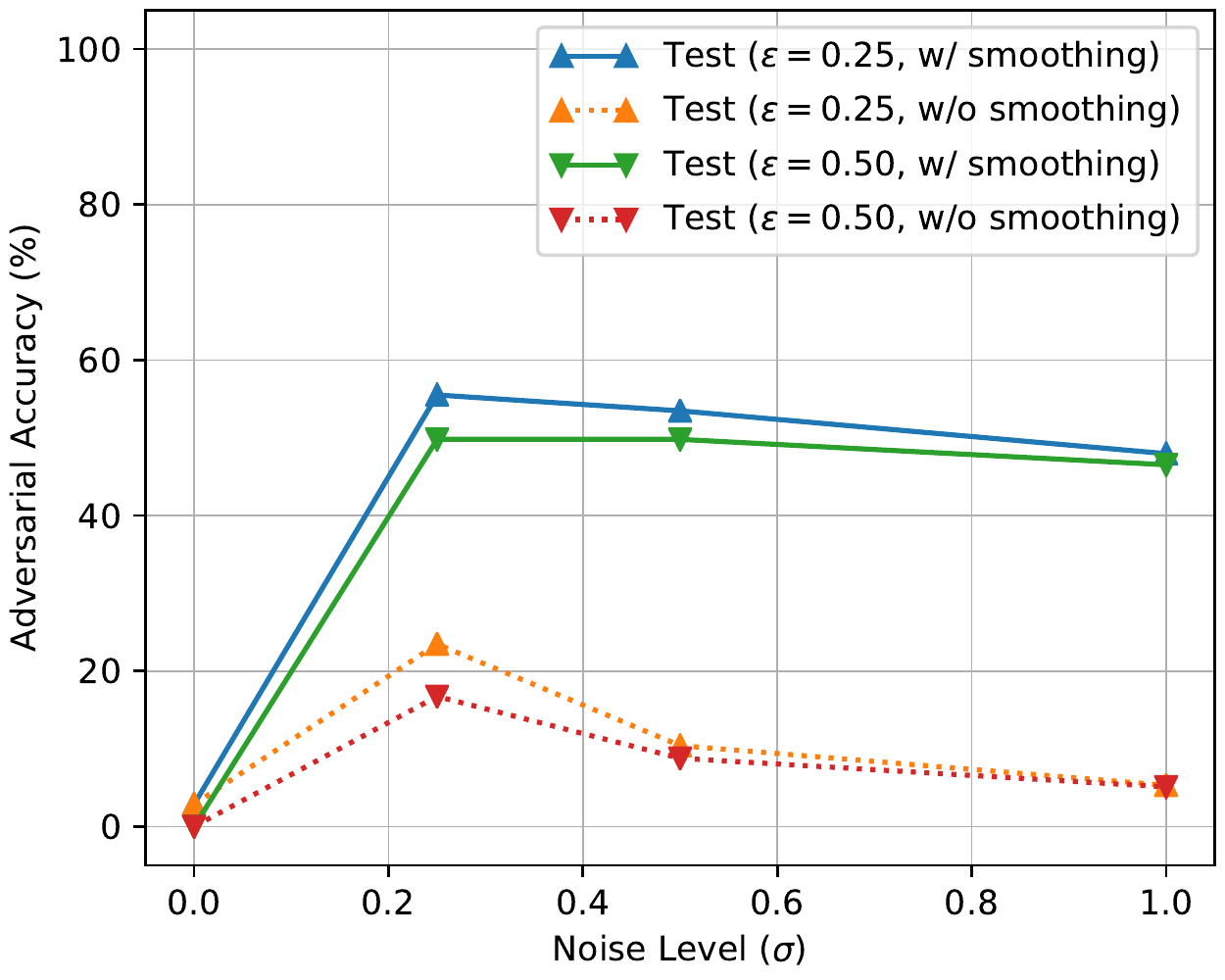}
}
\caption{Natural and Adversarial Accuracy (\%) vs Noise Levels ($\sigma$) for Randomized Smoothing with/without Smoothing (all models are trained with noise augmentation). Dotted lines refer to the corresponding performance when the model is trained with noise augmentation but not smoothed.}
\label{fig:exp/rand/compare}
\end{figure*}

Second, we consider models that are trained with noise augmentation but not smoothed. We use a normal PGD attack to evaluate the adversarial accuracy for the non-smoothed models. \Cref{fig:exp/rand/compare} compares the performance of these non-smoothed models with smoothed models. As evident in \cref{fig:exp/rand/compare}, we see that \textbf{noise augmentation provides adversarial robustness comparable to that of the randomized smoothing operation, especially for relatively small noise levels}. For instance, the adversarial accuracy for smoothed and non-smoothed models in \cref{fig:exp/rand/mnist/adversarial,fig:exp/rand/gtsrb/adversarial} almost overlap at either small or large noise levels. \Cref{fig:exp/rand/cifar/natural} also shows a less significant but consistent trend.
%

\subsection{Q2. Impact of Noise Parameter}\label{sec:exp/detail/eval}

In this section, we explore the impact of the noise parameter on the classifier's performance. In the following set of experiments, the smoothed model is always trained with the correct noise augmentation, as specified in \cite{DBLP:conf/icml/CohenRK19}.
In addition to the natural and adversarial accuracy, we also characterize the performance of randomized smoothing in terms of a proxy of the estimation error. We define this proxy as the difference between the natural accuracy on the training set and the test set, i.e., the accuracy that we lose when applying a model to the test set.

As shown in \cref{fig:exp/eval}, we evaluate the natural and adversarial accuracy on four datasets by varying the noise level. The noise levels for each dataset are chosen to show the complete behavior of randomized smoothing. The adversarial accuracy under four attack levels are reported for each dataset.
Note that we did not explicitly calculate this proxy of estimation error, as it can be easily observed by looking at the gap between the curves of natural accuracy on the training and test sets, i.e., the gap between curves with \emph{round} and \emph{square} markers.

We now interpret \cref{fig:exp/eval} from three perspectives: natural accuracy, estimation error, and adversarial accuracy.

We begin by looking at the curves for natural accuracy, i.e., curves with \emph{round} and \emph{square} markers, which show that the natural accuracy consistently drops as the model is smoothed with higher noise levels. 
%

Next, we examine the performance of randomized smoothing from the perspective of a proxy of estimation error, i.e., the gap between the two curves with \emph{square} and \emph{round} markers. For all datasets, we see that \textbf{while smoothing comes at a cost of natural accuracy, the difference between accuracy on the test set  and the accuracy  on the training set decreases as the noise level increases.} The test set accuracy may even outperform the training set accuracy, for instance in \cref{fig:exp/eval/mnist}.

Finally, we study the adversarial accuracy that characterizes the effectiveness of adversarial robustness provided by randomized smoothing (see curves with \emph{triangular} markers). For a noise level of zero -- no smoothing -- the adversarial attack significantly reduces the accuracy as expected. At higher noise levels, randomized smoothing is more effective and starts providing adversarial robustness; the accuracy under adversarial attack increases and gradually approaches the natural one. However, the natural accuracy on the test set is also decreasing as the adversarial one trying to approach it, the adversarial accuracy is therefore always lower than the sharply decreasing natural accuracy, in the sense that \textbf{randomized smoothing fails to provide a satisfactory adversarial robustness before it costs too much test accuracy.} We also observe that the adversarial accuracy declines after the noise level is raised beyond a threshold. This observation supports our conclusion in \cref{thm:threshold}.


\begin{figure}[t]
\centering
\hspace{-1em}
\subfigure[MNIST]{
    \label{fig:exp/eval/mnist}
    \includegraphics[width=0.23\linewidth]{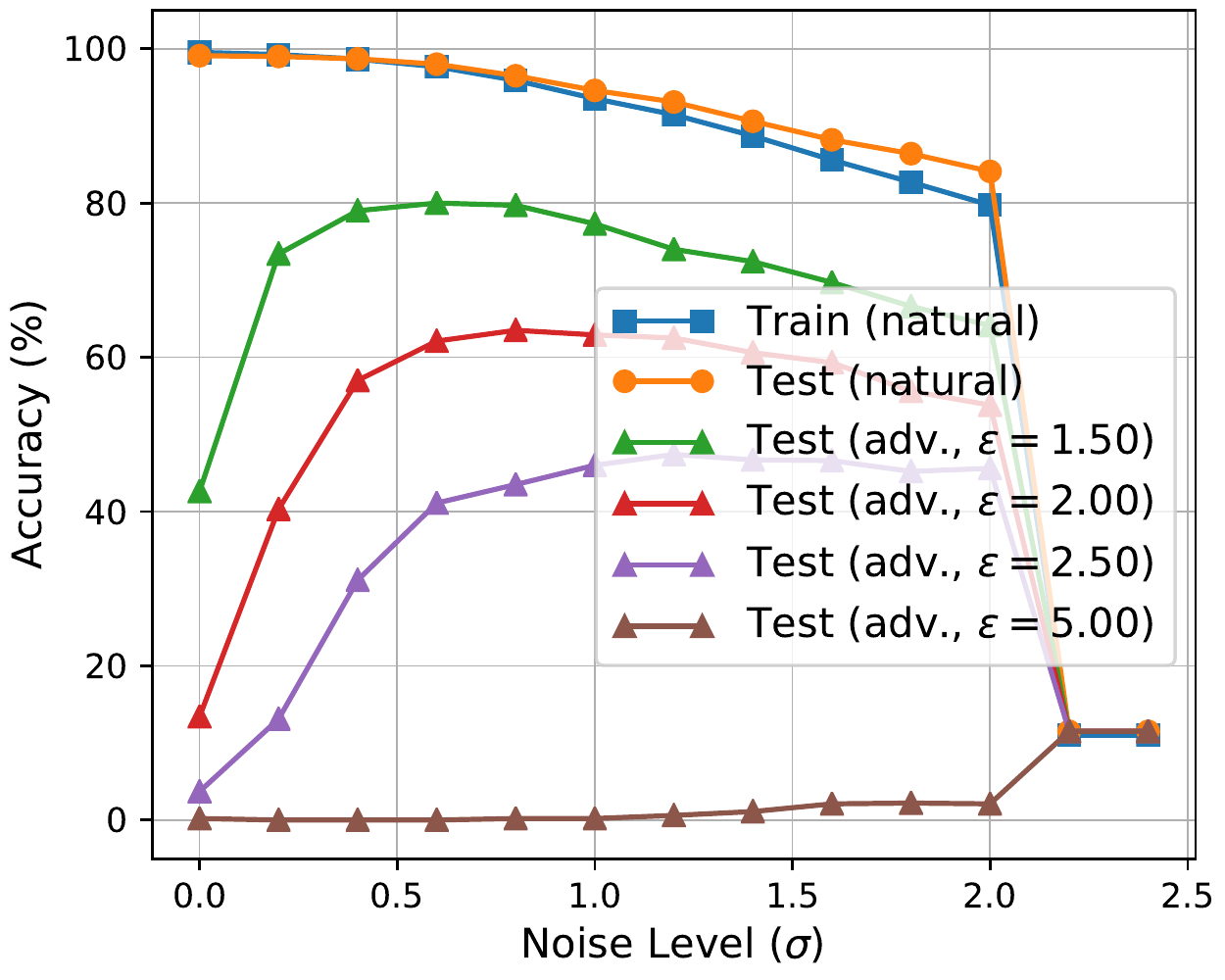}
}
\subfigure[CIFAR-10]{
    \label{fig:exp/eval/cifar}
    \includegraphics[width=0.23\linewidth]{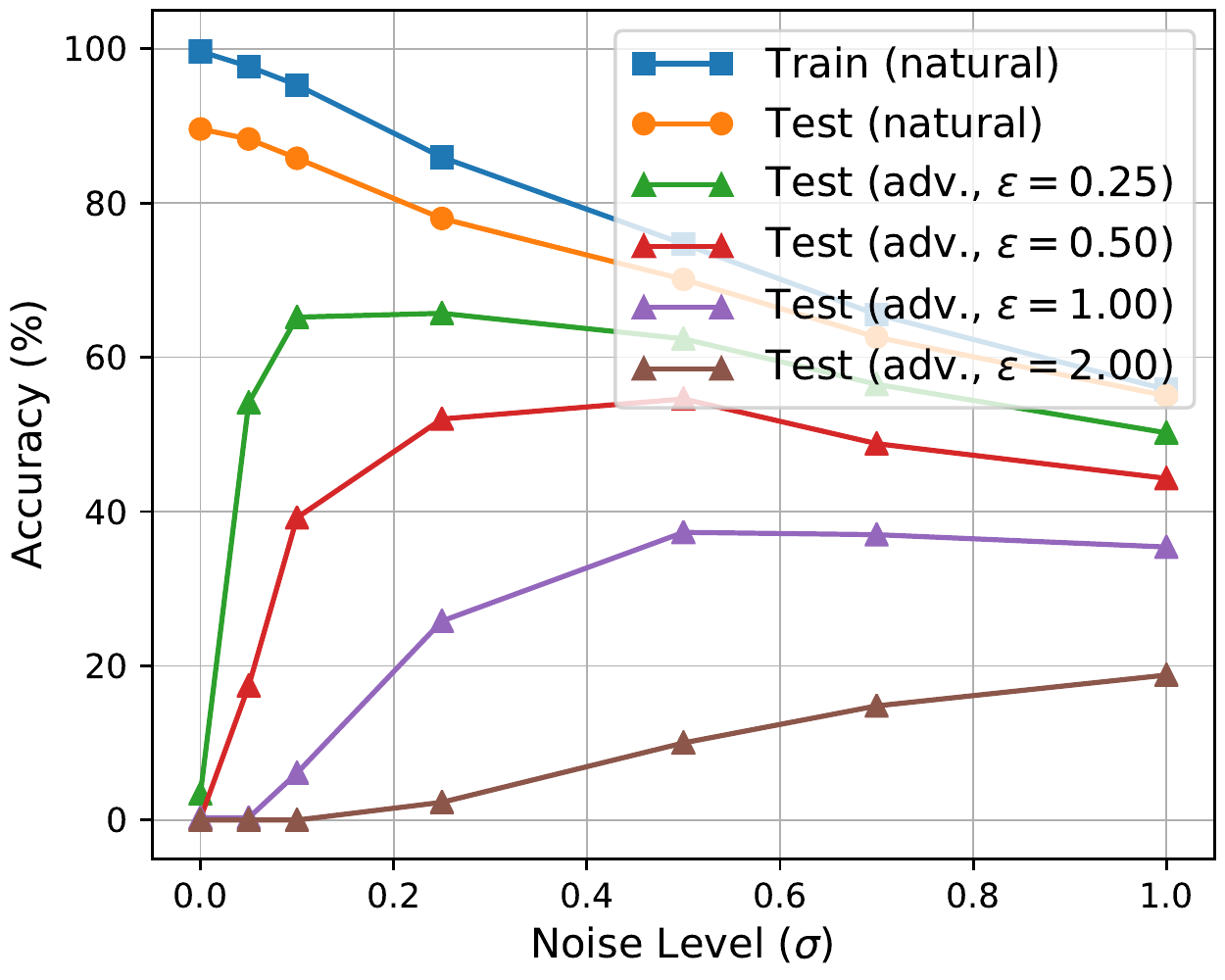}
}
\subfigure[GTSRB]{
    \label{fig:exp/eval/gtsrb}
    \includegraphics[width=0.23\linewidth]{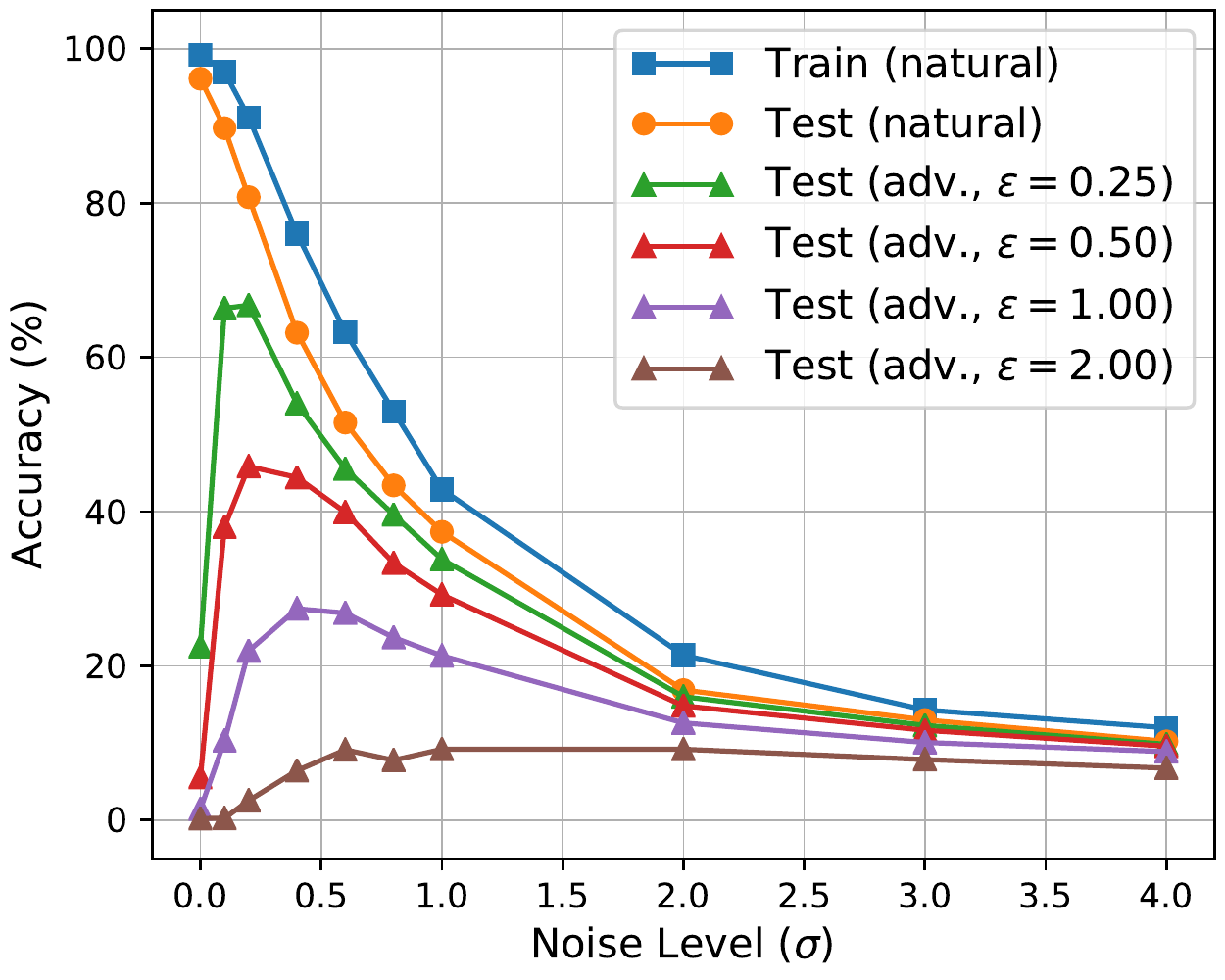}
}
\subfigure[ImageNet]{
    \label{fig:exp/eval/imagenet}
    \includegraphics[width=0.23\linewidth]{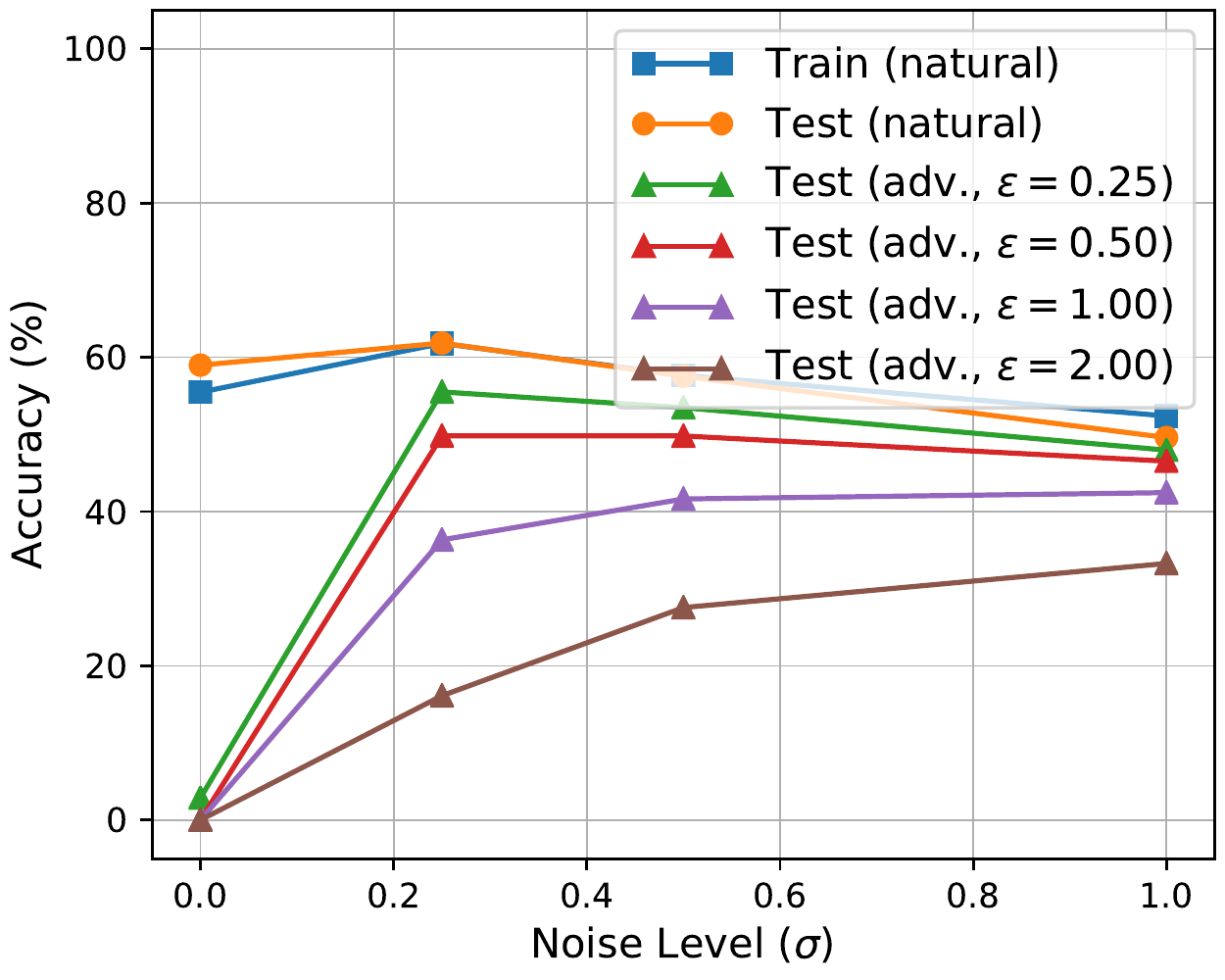}
}
\caption{Natural and Adversarial Accuracy (\%) vs Noise Levels ($\sigma$) for Randomized Smoothing. Triangular markers denote the adversarial accuracy under PGD attack with the maximum $\ell_2$ norm of $\epsilon$.}
\label{fig:exp/eval}
\end{figure}

\subsection{Q3. Size of the Hypothesis Class}
\label{sec:exp/detail/random}

In this section, we explore how the theory within this paper applies to general datasets.  Motivated by Zhang et al~\cite{random-labeling}, we use a set of MNIST examples, to each of which we have assigned a label uniformly at random from $\s{0,\dotsc,9}$, as a proxy for general datasets.  After constructing this new dataset, we train a model with the noise augmentation procedure at noise level $\scalarParameter$ and noise distribution $\noiseDistParametrized{}$, then we evaluate the performance of randomized smoothing. The model and parameters we used were specified in \cref{sec:exp/setup/network}. With our model trained using the noise augmentation procedure, we plot the accuracy of the classifier with and without the randomized smoothing operation applied, on both the test and train data sequences. When plotting the accuracy of a classifier to which the randomized smoothing operation has been applied, we use the same $\scalarParameter$ and noise distribution $\noiseDistParametrized{}$ as was applied in our noise augmented training procedure.  This entire process is repeated for 30 trials.  Moving forward, when the term original MNIST, when used without a modifier, means the original MNIST data, with original (non-random) labels.

When $\scalarParameter$ exceeds a threshold value $\scalarParameterVar$, our theory says the set of smooth hypotheses $\hypothesisClassSmooth$ realizable on $\unlabeledTrainingData{}$ will decrease, thereby reducing, in expectation, the accuracy of the best smooth hypothesis in that class.

It is well-known that clustering-based unsupervised learning techniques achieve respectable accuracy on original MNIST data.  For the original MNIST dataset, examples within the same class tend to lie close to one another.  With noise augmentation on a set of randomly labeled data, this clustering behavior to continues: As the noise level $\scalarParameter$ increases, larger clusters of points in feature space tend to have the same label, thereby reducing training accuracy on randomly-labeled MNIST.  Furthermore, the training accuracy on data when the randomized smoothing operation is applied, is closer to the test accuracy.  This may be evidence for why smoothing and noise augmentation often appear together.
%

Each model is trained for $200$ epochs with an initial learning rate of $0.01$, which is decreased by a factor of $0.95$ per epoch. We empirically observe that models trained with noise augmentation have higher loss, as shown in \cref{fig:exp/rand/converge}. Note that the model with zero noise arrives achieves low loss within the first few epochs. Such behavior could indicate the effect of noise augmentation.

The violin plot in \cref{fig:exp/random} shows that a model trained with noise augmentation is a better provides a better, though still sub-optimal, estimate of the natural accuracy.
Note first that the test set accuracy is always around $10\%$. As labels are assigned uniformly at random to our MNIST examples, the expected value of test accuracy is $10\%$, which is expected as the training and test sets are uncorrelated once we re-assign uniform random labels, thus the prediction on the test set is no better than a random guess over 10 classes.

\begin{figure}[h]
\centering
\begin{minipage}{0.45\linewidth}
    \centering
    \includegraphics[width=0.9\linewidth]{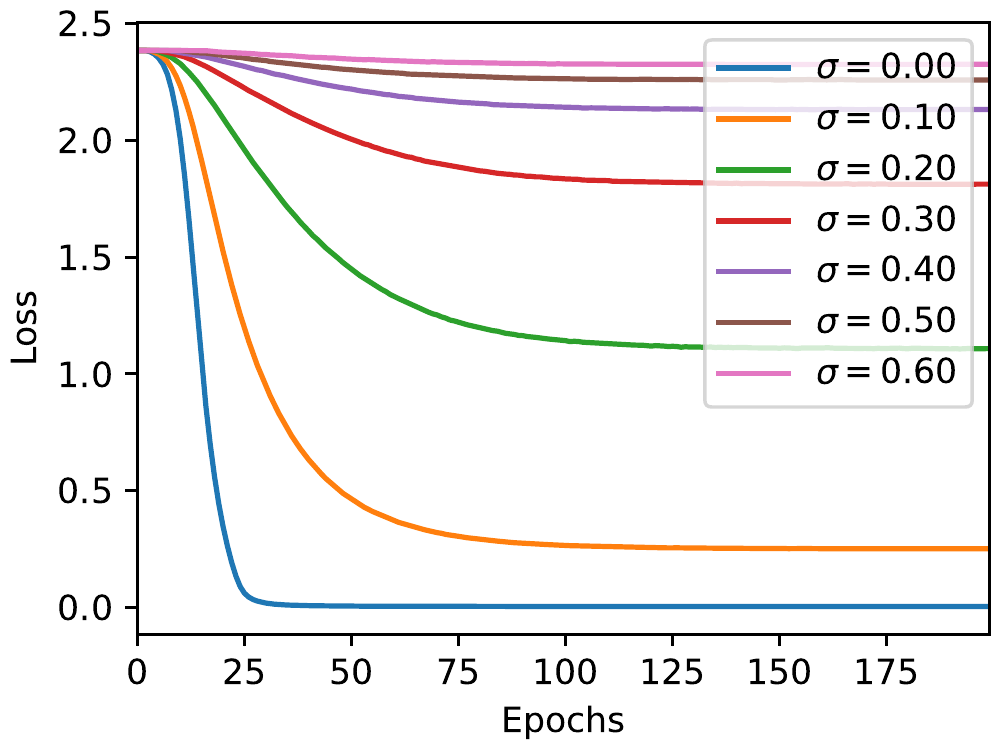}
    \caption{Averaged Training Loss vs Training Epochs for Different Noise Levels.}
    \label{fig:exp/rand/converge}
\end{minipage}
\hfill
\begin{minipage}{0.45\linewidth}
    \centering
    \includegraphics[width=0.9\linewidth]{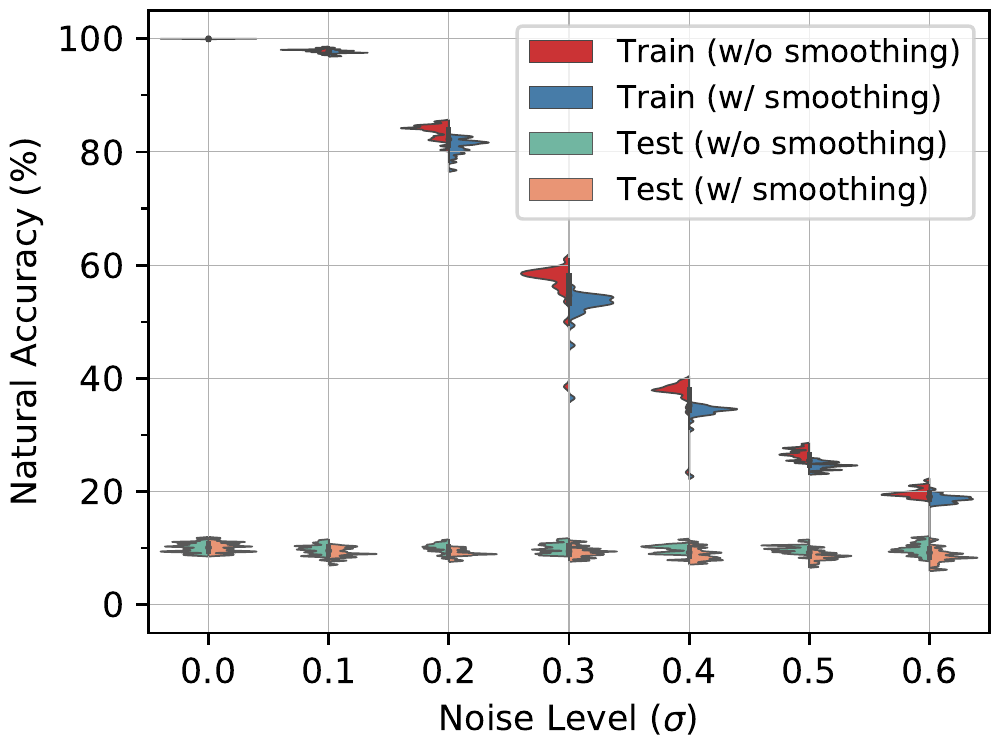}
    \caption{Natural Accuracy (\%) vs Noise Levels ($\sigma$) on Randomly-labeled MNIST}
    \label{fig:exp/random}
\end{minipage}
\end{figure}

We now focus on the training set accuracy;  note that the model with zero noise level achieves $100\%$ accuracy on the training set for all enumerated instances of labeling. This suggests that there exists a realizable hypothesis within the hypothesis class, before smoothing.
However, we see that the obtained accuracy decreases consistently as the model is trained with larger noise levels, and the estimation error also decreases as the test set accuracy remains the same. 

Apparently, smoothing should not work in this scenario, as close points are not expected to have the same label. So the training set accuracy with smoothing (blue) is always worse than that without smoothing (red).
We can make the same observations as before, even when smoothing is making the wrong decisions. These results suggest that \textbf{noise augmentation indeed reduces the size of the realizable hypothesis class, in the sense that this procedure makes it harder for the model to realize a given hypothesis}, as we have discussed in \cref{sec:labelClustering}.
Therefore, we base these implications on noise augmentation, and in particular, its impact on reducing the realizable hypothesis class.

To summarize, noise augmentation, we have empirically verified \cref{thm:threshold}.  Furthermore, we observe that not only does smoothing reduce the size of the hypothesis class, but we also notice a relationship between the statistical dispersion, and the size of $\hypothesisClassSmooth$.

\section{Related Work}
\label{sec:related}

Our work mainly relates to the recent research in randomized smoothing. We also compare against other works that analyze noise augmentation in the aspects of adversarial robustness.

\subsection{Randomized Smoothing}
Randomized smoothing is a procedure that produces a new classifier from an existing base classifier. It provides certifiable adversarial robustness by leverages the base classifier's robustness to random noise. Several earlier works proposed randomized smoothing as a heuristic defense without providing any robustness guarantees, for instance from Cao et al.~\cite{region_based} and Liu et al.~\cite{self_ensemble}.
Lecuyer et al.~\cite{DBLP:conf/sp/LecuyerAG0J19} proved the first robustness guarantee for randomized smoothing, utilizing Laplacian noise and its well-studied inequalities from the differential privacy literature. Subsequently, Cohen et al.~\cite{DBLP:conf/icml/CohenRK19} summarized these approaches as randomized smoothing and provided a tighter robustness guarantee with Gaussian noise.

Confirming the tightness of a robustness guarantee from Gaussian noise, Blum et al. recently showed that, for $\ell_p$ norm-bounded adversaries, $p > 2$, on high dimensional images where pixel values are bounded between $0$ and $255$, the noise comes to dominate any useful information in the images, thereby leading to trivial smoothed classifiers \cite{blum2020random}.

Recently, several works highlighted the importance of the training procedure for randomized smoothing. Salman et al.~\cite{smoothing-pgd} proposed the first PGD attack for randomized smoothing and thus improved the robustness guarantee via adversarial training. Similarly, Li et al.~\cite{stabilityTraining_smoothing} improved this guarantee by introducing stability training to improve the base model's robustness to Gaussian noise.

However, all of these works did not explicitly or formally explore the relationship between the training of the base model and the following smoothing procedure. Specifically, they did not explain why randomness has to be used in both the training and inference time. The subsequent implications on adversarial robustness and classification accuracy were not explained either.
For instance, Lecuyer et al.~\cite{DBLP:conf/sp/LecuyerAG0J19} empirically pointed out that the noise should also be added in the training phase. Cohen et al.~\cite{DBLP:conf/icml/CohenRK19} also empirically observed that the same noise level should be used in the training and inference time. We build on both works by formally exploring  the relationship between noise augmentation and randomized smoothing. We also consider adversarial accuracy and generalization ability.
Despite the improved robustness guarantee from Salman et al.~\cite{smoothing-pgd} and Li et al.~\cite{stabilityTraining_smoothing}, they did not formally explain the connection between the training procedure and smoothing like we do.

Earlier work from Liu et al.~\cite{self_ensemble} established the equivalence between the training with noise (at each layer) and Lipschitz regularization, but this explanation does not apply to the implications of randomized smoothing. In contrast, our explanation in the aspects of the realizable hypothesis class covers both the connection and its implications.
Cao et al.~\cite{region_based} claimed that the noise only needs to be used in the test phase. This happens when the added noise is small and below the threshold that we found.

We emphasize that we do not compare against randomized smoothing approaches in the aspects of providing robustness guarantees. Instead, we seek to highlight and explain a connection between the training and smoothing procedure, as well as its implications on adversarial robustness and classification accuracy.

\subsection{Noise Augmentation}

Indeed, the training with added noise and its robustness properties have been well studied in the noise augmentation literature, but not in the context of randomized smoothing. For instance, Fawzi et al.~\cite{robustness_noise} and Franceschi et al.~\cite{robustness_gaussian} related the robustness of classifiers from random noise to adversarial examples. Zantedeschi et al.~\cite{gaussian_augmentation} introduced Gaussian data augmentation and observed that the added noise sometimes even improve the accuracy.

The above works were referenced in the randomized smoothing literature, but were not sufficient to explain the connection between noise augmentation and randomized smoothing. In contrast, we provide an explanation in the aspects of statistical machine learning theory, which covers noise augmentation, its connection to randomized smoothing, and the subsequent implications on adversarial robustness and natural accuracy. Moreover, the critical threshold that we found for randomized smoothing explains the observation from Zantedeschi et al.~\cite{gaussian_augmentation} that the added noise can sometimes improve the accuracy.

\section{Conclusion}
\label{sec:conclusion}
Randomized smoothing is one of the promising defenses against adversarial attacks on classifiers. The randomized smoothing technique
smoothes a classifier's prediction by adding random noise to the data point. In this paper, we theoretically
and empirically explore randomized smoothing. We investigate the effect of randomized smoothing on the space 
of realizable hypothesis space, and show that for some noise levels the realizable
hypothesis space shrinks due to smoothing. This result could be a potential explanation why the natural accuracy drops due to smoothing. 
We perform extensive experiments to empirically support our theoretical investigation using well-known image classification datasets.

\newpage
\bibliographystyle{plain}
\bibliography{src/essential/ref,src/essential/ref-eval,src/essential/ref-related-work}

\end{document}